\documentclass[bj,authoryear]{imsart}
\makeatletter%%%%%to delete 
\def\journal@name{}%%%%%to delete
\makeatother%%%%%to delete

\usepackage{amsmath,amssymb,mathtools}
\usepackage{mathrsfs}
\usepackage{bm}

\usepackage{graphicx}

\usepackage{booktabs}
\usepackage{array}
\usepackage{tabularx}
\usepackage{ragged2e}

\usepackage{comment}
\usepackage{placeins}

\startlocaldefs

\numberwithin{equation}{section}

\theoremstyle{plain}
\newtheorem{theorem}{Theorem}[section]
\newtheorem{proposition}[theorem]{Proposition}
\newtheorem{lemma}[theorem]{Lemma}
\newtheorem{corollary}[theorem]{Corollary}

\theoremstyle{definition}

\newtheorem{remark}[theorem]{Remark}

\newcounter{Pass}
\newcommand{\Pitem}[2]{%
  \par\refstepcounter{Pass}\noindent
  \textbf{(P\thePass) #1.}\ #2}

\newcounter{Lass}
\newcommand{\Litem}[2]{%
  \par\refstepcounter{Lass}\noindent
  \textbf{(L\theLass) #1.}\ #2}

\newcolumntype{Y}
    {>{\RaggedRight\hspace{0pt}\arraybackslash}X}

\newcolumntype{P}[1]
    {>{\RaggedRight\hspace{0pt}\arraybackslash}p{#1}}

\newcolumntype{L}[1]
    {>{\raggedright\arraybackslash}p{#1}}

\graphicspath{{figures/}{numerics/}}

\endlocaldefs

\begin{document}

\begin{frontmatter}

\title{Weighted Empirical Risk Minimization for Machine Learning under Long-Range Dependence: Exact Pathwise Rates and Learning-Error Geometry}

%\runtitle{Weighted LIL for Long-Memory Learning}

\begin{aug}

\author[A]{
\inits{E.}
\fnms{Elina}~
\snm{Moldavskaya}
\ead[label=e1]{elina.mol@technion.ac.il}
}

\address[A]{
Technion--Israel Institute of Technology, Israel
\printead[presep={,\ }]{e1}
}

\end{aug}

\begin{abstract}
We develop an exact almost-sure learning theory for smooth parametric
models trained by regularly weighted empirical risk minimization on
long-range dependent data. The training observations are generated
from a fixed finite window of a stationary Gaussian sequence, and the
sample weights are regularly varying. If the loss gradient at the
population minimizer has Wiener-chaos rank \(m\) and a nonzero
low-frequency coefficient, then, in the long-memory interior regime,
the finite-lag score reduces on the iterated-logarithm scale to a
single weighted Hermite chaos. This yields an almost-sure Bahadur
representation, an exact limsup law for the learned parameter, and,
for \(m\ge2\), the functional cluster set of the complete learning
trajectory. The polynomial learning exponent is determined by the
memory parameter and the chaos rank and is invariant under the
admissible power weighting, whereas the sharp pathwise constant and
cluster geometry depend on the weights. In the rank-one case, global
optimization over the admissible power exponents shows that every
optimizer is positive. Time-series prediction and classification
examples illustrate the results.
\end{abstract}


\end{frontmatter}

%% ================================================================
%% Main text
%% ================================================================

\section{Introduction}
\label{sec:introduction}

Machine-learning methods are often trained on temporally ordered data,
including time series, sensor streams, network measurements, and
repeatedly observed systems.  A substantial learning-theoretic
literature treats dependent observations through mixing and related
weak-dependence conditions and derives finite-sample generalization
guarantees; see, for example,
\cite{MohriRostamizadeh2010,KuznetsovMohri2017}.  Long-range dependence
has a different asymptotic structure: correlations are not summable,
and cumulative fluctuations may be substantially larger and more
persistent than under independent or short-memory sampling.

Sample weighting is also common in machine learning, for example under
covariate shift and importance weighting
\cite{Shimodaira2000,SugiyamaKrauledatMueller2007,CortesMansourMohri2010}.
Here we consider a different class: deterministic regularly varying
sample weights
\(
    a_t=t^{-\kappa}L_a(t), \quad t=1,2,\ldots,
\)
where \(\kappa\in\mathbb R\) is the power exponent and \(L_a\) is
slowly varying at infinity.  

Our main question is whether such 
weighting changes the exact pathwise learning rate under long-range
 dependence and, if not, how it changes the sharp fluctuation constant.

We study a smooth parametric predictor \(f_\theta\) with parameter
vector \(\theta\in\Theta\subset\mathbb R^p\), trained from
observations \(Z_t=(U_t,Y_t)\), where \(U_t\) denotes the feature
vector and \(Y_t\) the response or label.  We assume that, for some
fixed lag \(r\ge0\),
\(
    Z_t
    =
    \Phi(X_t,X_{t-1},\ldots,X_{t-r}),
\)
for a fixed measurable map \(\Phi\), where \(\{X_t\}_{t\in\mathbb Z}\) is a centered stationary Gaussian
sequence with unit variance and covariance function
\(
    \rho(h)
    :=
    \mathbb E[X_0X_h]
    \sim
    c h^{-\alpha}L_0(h),
    \quad h\to\infty,
\)
where \(c>0\), \(0<\alpha<1\), and \(L_0\) is slowly varying at
infinity.  For a loss function \(\ell\), measuring the cost of a prediction (prediction error), we write
\[
    \ell(\theta;Z_t)
    =
    \ell\bigl(f_\theta(U_t),Y_t\bigr).
\]

Let \(W_n:=\sum_{t=r+1}^{n}a_t\).  The population and weighted
empirical risks are \(R(\theta):=\mathbb E\,\ell(\theta;Z_0)\) and
\(\widehat R_n(\theta):=W_n^{-1}\sum_{t=r+1}^{n}
a_t\ell(\theta;Z_t)\), respectively, and we denote the corresponding
population and empirical minimizers by \(\theta_0\) and
\(\widehat\theta_n\), respectively.
The case
\(a_t\equiv1\) gives ordinary empirical risk minimization.

The central probabilistic object is the loss gradient at the population
minimizer,
\[
g(X_t,\ldots,X_{t-r})
=
\nabla_\theta
\ell\!\left(\theta_0;\Phi(X_t,\ldots,X_{t-r})\right)
\in\mathbb R^p .
\]
Its first nonzero Wiener chaos determines the leading long-memory
fluctuations.

For nonlinear functionals of Gaussian long-memory sequences, the first
nonzero Wiener chaos determines the leading dependence; this is the
Hermite-rank mechanism underlying classical long-memory limit theory
\cite{Taqqu1977,DehlingTaqqu1989,Arcones1994}. We denote by \(m\) the Wiener-chaos rank of the loss gradient, that is,
the order of its first nonzero chaos, and by
\(\mathcal J\in\mathbb R^p\) its effective low-frequency coefficient.

The law of the iterated logarithm (LIL) describes the exact almost-sure scale of the largest recurrent fluctuations of partial sums. The probabilistic input of the present work is the weighted LIL established in the 
companion paper~\cite{Moldavskaya2026WeightedLIL}.  It studies weighted sums of nonlinear functionals of
long-memory Gaussian sequences. In particular, for the pure-chaos
reference statistic
\[
T_n=\sum_{t\le n}a_tH_m(X_t),
\]
where \(H_m\) is the \(m\)th
Hermite polynomial and \(B_n^2=\operatorname{Var}(T_n)\), the
companion result identifies the almost-sure scale
\(B_n(2\log\log n)^{m/2}\).  In the nonlinear regime it also gives the
functional cluster set and the corresponding variational LIL
constant.  The present paper develops the consequences of this
pathwise theory for statistical learning.

Classical almost-sure theory for \(M\)-estimators already includes
laws of the iterated logarithm and Bahadur-type invariance
representations; see \cite{HeWang1995}. For dependent errors,
\cite{FanYanXiu2014} established strong Bahadur representations and
an LIL for \(M\)-estimators in linear models with a Markov dependence
structure. These results do not address regularly weighted ERM in
the noncentral Gaussian long-memory regime or the chaos-rank-dependent
functional cluster geometry studied here.

Statistical estimation under long-range dependence has been studied
extensively, including M-estimation and least-squares estimation in
regression; see, for example, \cite{Koul1992,Moldavskaya2007LSE}, the
latter in a constrained regression setting with strong dependence.
Weighted nonlinear Gaussian functionals with singular spectra were
studied by \cite{IvanovEtAl2013}, who obtained Gaussian distributional
limits under general weight conditions.  Our question is different:
we seek exact almost-sure normalizations, limsup constants, and
functional cluster geometry in the noncentral long-memory regime.
Also related is the nonasymptotic anytime LIL for general M-estimators
of \cite{SchreuderBrunelDalalyan2020}; that work provides finite-time
probability guarantees, whereas we study exact pathwise behavior
generated by long memory and Wiener-chaos structure.
Very recently, \cite{BenningNourdinPeccati2026} showed that
correlation decay and Hermite rank also determine the infinite-depth
limit of residual networks with correlated initialization. Their
dependence acts across network layers at initialization, whereas here
long-range dependence is generated by the temporally dependent training
data and we study exact almost-sure learning fluctuations.

The main consequences are as follows.

\begin{enumerate}
\item
We prove a finite-lag reduction for the loss gradient.  Although it
depends on the Gaussian vector
\((X_t,\ldots,X_{t-r})\), on the leading LIL scale the effect of the
fixed lag window is reduced to the coefficient \(\mathcal J\):
\begin{equation}
\label{eq:intro-finite-lag}
    \sum_{t=r+1}^{n}
    a_tg(X_t,\ldots,X_{t-r})
    =
    \mathcal J
    \sum_{t=1}^{n}a_tH_m(X_t)
    +
    o_{\mathrm{a.s.}}
    \bigl(B_n(2\log\log n)^{m/2}\bigr).
\end{equation}
This reduction connects the companion scalar LIL with the learning
problem.

\item
Under standard smoothness, identifiability, and positive-curvature
conditions, we obtain a pathwise asymptotic linear representation (of Bahadur type) that expresses
the learning error through the weighted score:
\begin{equation}
\label{eq:intro-Bahadur}
    \widehat\theta_n-\theta_0
    =
    -W_n^{-1}V^{-1}\mathcal J
    \sum_{t=1}^{n}a_tH_m(X_t)
    +
    o_{\mathrm{a.s.}}(r_n),
\end{equation}
where \(V=\nabla^2R(\theta_0)\) and
\(r_n=B_n(2\log\log n)^{m/2}/W_n\).  This yields exact endpoint limsup laws and, for \(m\ge2\), the full
functional cluster set describing the normalized learning trajectory.

\item
The leading multidimensional learning error lies in the deterministic
direction \(V^{-1}\mathcal J\).  Thus scalar Gaussian finite-lag
subordination produces a one-dimensional leading pathwise geometry
even when \(\theta\in\mathbb R^p\).

\item
In the interior regime
\(\alpha m<1\) and \(2\kappa+\alpha m<1\),
\[
    r_n
    \asymp
    n^{-\alpha m/2}
    L_0(n)^{m/2}
    (2\log\log n)^{m/2}.
\]
More precisely, the slowly varying weight factor \(L_a\) cancels and
\(\kappa\) does not affect the polynomial exponent, although it
changes the sharp LIL constant and, in the nonlinear case, the cluster
geometry.  The excess population risk satisfies an exact LIL on the corresponding
squared scale, with the sharp constant identified explicitly.

\item
For \(m=1\), the sharp weighting problem can be solved globally over
the full admissible range
\(\kappa<(1-\alpha)/2\).  Every global optimizer is strictly positive:
both negative exponents and ordinary unweighted ERM are suboptimal for
the exact rank-one LIL constant.

\end{enumerate}

The results are illustrated by three machine-learning examples.
One-step linear prediction has chaos rank two despite its linear
parameterization; lagged threshold classification has chaos rank one;
and a quadratic-target model exhibits a symmetry-induced transition
between chaos ranks one and two.  Reproducible numerical experiments
illustrate the effect of weighting on the sharp LIL constant, the
one-dimensional geometry of the learning error, and the
symmetry-induced rank transition. The numerical study also reveals a pronounced asymmetry: optimal positive weighting gives only a modest improvement, whereas weighting in the opposite direction can be substantially more costly.
%The paper is organized as follows.
%Section~\ref{sec:model} introduces the learning model and assumptions.
%Section~\ref{sec:score-reduction} proves the finite-lag reduction.
%Sections~\ref{sec:parameter-LIL} and
%\ref{sec:functional-learning} establish the endpoint and functional
%learning LILs.
%Section~\ref{sec:rates} derives explicit rates, excess-risk
%asymptotics, and the weighting results.
%Sections~\ref{sec:examples} and \ref{sec:numerics} give examples and
%numerical illustrations, and Section~\ref{sec:discussion} concludes.

Sections~\ref{sec:model}--\ref{sec:score-reduction} develop the learning
model and finite-lag reduction, Sections~\ref{sec:parameter-LIL}--\ref{sec:rates}
establish the learning LILs and study weighting,
Sections~\ref{sec:examples}--\ref{sec:numerics} give examples and numerical
experiments, and Section~\ref{sec:discussion} discusses the scope and
extensions.
% =====================================================================
% Section 2. Weighted empirical risk minimization for long-memory data
% =====================================================================

\section{Weighted empirical risk minimization for long-memory data}
\label{sec:model}

\subsection{Learning model}
\label{subsec:learning-model}

Let \(\{X_t\}_{t\in\mathbb Z}\) be a centered stationary Gaussian
sequence with unit variance and covariance
\(\rho(h)=\mathbb E[X_tX_{t+h}]\).  For a fixed integer \(r\ge0\), put
\(\mathbf X_t^{(r)}:=(X_t,X_{t-1},\ldots,X_{t-r})\) and
\(Z_t=(U_t,Y_t):=\Phi(\mathbf X_t^{(r)})\).
Thus the model includes finite-window time-series prediction; for
\(r\ge1\), for example,
\(U_t=(X_{t-1},\ldots,X_{t-r})\) and \(Y_t=X_t\).

Let \(f_\theta\), \(\theta\in\Theta\subset\mathbb R^p\), be a
parametric predictor and
\(\ell(\theta;Z_t)=\ell(f_\theta(U_t),Y_t)\) its loss.
Define \(R(\theta):=\mathbb E\,\ell(\theta;Z_0)\) and
\(\theta_0:=\operatorname*{arg\,min}_{\theta\in\Theta}R(\theta)\).

For regularly varying weights
\(a_t=t^{-\kappa}L_a(t)\), \(\kappa\in\mathbb R\), let
\(W_n:=\sum_{t=r+1}^{n}a_t\) and \\
\(\widehat R_n(\theta):=W_n^{-1}\sum_{t=r+1}^{n}
a_t\ell(\theta;Z_t)\), and choose
\(\widehat\theta_n\in
\operatorname*{arg\,min}_{\theta\in\Theta}\widehat R_n(\theta)\).
The unweighted case \(a_t\equiv1\) is included as a benchmark; our
interest is in how nonconstant power weights affect the exact
long-memory learning fluctuations.

\subsection{Loss gradient and Wiener-chaos rank}\label{Loss gradient}

Write \(\psi(\theta;z):=\nabla_\theta\ell(\theta;z)\) and
\(g(\mathbf x):=\psi(\theta_0;\Phi(\mathbf x))\).
Under the assumptions below,
\(\mathbb E\,g(\mathbf X_0^{(r)})=\nabla R(\theta_0)=0\).

Let \(W\) be an isonormal Gaussian process over a real Hilbert space
\(\mathfrak H\) such that \(X_t=W(h_t)\) and
\(\langle h_s,h_t\rangle_{\mathfrak H}=\rho(t-s)\).
Put \(E_t:=\operatorname{span}\{h_t,\ldots,h_{t-r}\}\).
Since \(g(\mathbf X_t^{(r)})\) is
\(\sigma(X_t,\ldots,X_{t-r})\)-measurable, each coordinate of its
\(q\)th chaos kernel belongs to the symmetric tensor power
\(E_t^{\odot q}\). Hence, coordinatewise,
\(g(\mathbf X_t^{(r)})=\sum_{q\ge1}I_q(K_{q,t})\), where \(I_q\)
denotes the \(q\)th multiple Wiener--It\^o integral and
\begin{equation}
\label{eq:finite-lag-chaos}
K_{q,t}
=
\sum_{\mathbf j\in\{0,\ldots,r\}^q}
\beta_{\mathbf j}^{(q)}
\,h_{t-j_1}\widetilde\otimes\cdots
\widetilde\otimes h_{t-j_q},
\qquad
\beta_{\mathbf j}^{(q)}\in\mathbb R^p .
\end{equation}
Here \(\widetilde\otimes\) denotes the symmetrized tensor product.

Define the chaos rank by
\(m:=\min\{q\ge1:K_{q,0}\neq0\}\).
Under \textup{(P\ref{P1})} below, define the leading low-frequency
coefficient intrinsically by
\[
\mathcal J
:=
\lim_{h\to\infty}
\frac{\operatorname{Cov}\bigl(
g(\mathbf X_0^{(r)}),H_m(X_h)\bigr)}
{m!\rho(h)^m}
\in\mathbb R^p ,
\]
where the covariance is understood componentwise.
Indeed, since \(H_m(X_h)=I_m(h_h^{\otimes m})\) and\\
\(\langle f_1\widetilde\otimes\cdots\widetilde\otimes f_m,
h^{\otimes m}\rangle
=\prod_{\nu=1}^m\langle f_\nu,h\rangle\), orthogonality of the
Wiener chaoses and the Wiener--It\^o isometry give
\[
\operatorname{Cov}\bigl(
g(\mathbf X_0^{(r)}),H_m(X_h)\bigr)
=
m!\sum_{\mathbf j\in\{0,\ldots,r\}^m}
\beta_{\mathbf j}^{(m)}
\prod_{\nu=1}^{m}\rho(h+j_\nu).
\]
Since \(\rho(h+j)/\rho(h)\to1\) for each fixed \(j\) by
\textup{(P\ref{P1})}, the limit exists and
\(\mathcal J=\sum_{\mathbf j\in\{0,\ldots,r\}^m}
\beta_{\mathbf j}^{(m)}\).
In particular, \(\mathcal J\) is independent of the particular tensor
representation in \eqref{eq:finite-lag-chaos}.

When \(\mathcal J=0\), the finite-lag combination cancels the leading
zero-frequency contribution, analogously to differencing a
long-memory series; this degenerate case is excluded in
\textup{(P\ref{P2})}. For \(r=0\), if
\(J_m:=\mathbb E[g(X_0)H_m(X_0)]\in\mathbb R^p\), then
\(\mathcal J=J_m/m!\), so the definitions reduce to the usual
Hermite expansion and Hermite rank.

\subsection{Probabilistic assumptions}
\label{subsec:prob-assumptions}

\Pitem{Long-range dependence}{\label{P1}%
\(\rho(h)\sim c h^{-\alpha}L_0(h)\) as \(h\to\infty\), where
\(c>0\), \(0<\alpha<1\), and \(L_0\) is eventually positive and
slowly varying.}

\Pitem{Finite-lag chaos rank}{\label{P2}%
Let \(\mu_r\) denote the law of \(\mathbf X_0^{(r)}\).
The function \(g\in L^2(\mu_r;\mathbb R^p)\) has zero mean, chaos
rank \(m\ge1\), and leading low-frequency coefficient
\(\mathcal J\neq0\).}

\Pitem{Weights and standing regime}{\label{P3}%
The weights satisfy \(a_t=t^{-\kappa}L_a(t)\), \(\kappa\in\mathbb R\),
where \(L_a\) is eventually positive and slowly varying. Throughout,
the standing regime is \(\alpha m<1\) and
\(2\kappa+\alpha m<1\). Hence \(\kappa<1/2\); in particular,
\(\kappa<1\), and Karamata's theorem~\cite{BGT1987} gives
\(W_n\sim n^{1-\kappa}L_a(n)/(1-\kappa)\).}

Since modifying finitely many weights does not affect any of the
asymptotic results, we assume without loss of generality that
\(a_t>0\) for all \(t\ge1\).

For \(m\ge2\), the nonlinear functional results additionally use the
following assumptions from the companion weighted-LIL paper.

\Pitem{Spectral long-memory condition}{\label{P4}%
The sequence admits a spectral density \(f\) satisfying
\(\rho(h)=\int_{-\pi}^{\pi}e^{ih\lambda}f(\lambda)\,d\lambda\) and
\(f(\lambda)\sim c_f|\lambda|^{\alpha-1}
L_0(1/|\lambda|)\) as \(\lambda\to0\), with \(c_f>0\);
moreover, \(f\) is bounded on every compact subset of
\((-\pi,\pi]\setminus\{0\}\).} 

When \textup{(P\ref{P1})} and \textup{(P\ref{P4})} hold with the same \(L_0\),
their constants satisfy
\(c=c_\alpha c_f\),
\(c_\alpha=2\Gamma(\alpha)\cos(\pi\alpha/2)\).

\Pitem{Weight regularity}{\label{P5}%
\(L_a\) is eventually continuously differentiable, with
\(xL_a'(x)/L_a(x)\to0\) as \(x\to\infty\).}

\Pitem{High-frequency spectral regularity}{\label{P6}%
For every \(\delta\in(0,\pi)\),
\(f\in W^{1,1}([-\pi,-\delta]\cup[\delta,\pi])\).}

For \(r=0\), \textup{(P\ref{P1})--(P\ref{P3})} reduce to the
corresponding time-domain, Hermite-rank, and weight assumptions of the
companion paper; \textup{(P\ref{P4})--(P\ref{P6})} are the additional
assumptions used there for the nonlinear functional LIL. The new
ingredient for \(r\ge1\) is the finite-lag reduction proved in
Section~\ref{sec:score-reduction}.

\subsection{Learning assumptions}
\label{subsec:learning-assumptions}

\Litem{Identifiability}{\label{L1}%
The set \(\Theta\subset\mathbb R^p\) is compact,
\(\theta_0\in\operatorname{int}\Theta\), and \(\theta_0\) is the
unique well-separated minimizer of \(R\): for every
\(\varepsilon>0\),
\(\inf_{\theta\in\Theta:\,\|\theta-\theta_0\|\ge\varepsilon}
R(\theta)>R(\theta_0)\).
A measurable minimizer \(\widehat\theta_n\) of \(\widehat R_n\) is
selected for every \(n\).}

\Litem{Smooth loss}{\label{L2}%
There exists an open neighborhood \(\Theta_0\) of \(\theta_0\) with
\(\Theta_0\subset\operatorname{int}\Theta\) such that, for almost
every \(z\) under the law of \(Z_0\),
\(\theta\mapsto\ell(\theta;z)\) is twice continuously differentiable
on \(\Theta_0\), the first two derivatives may be interchanged with
expectation, and
\(\nabla R(\theta_0)=\mathbb E\,\psi(\theta_0;Z_0)=0\).}

\Litem{\(L^2\)-envelopes}{\label{L3}%
There exist
\(\Xi_0,\Xi_\ell,\Xi_1,\Xi_2,\Xi_H\in L^2(\mu_r)\) such that
\(\sup_{\theta\in\Theta}|\ell(\theta;\Phi(\mathbf x))|
\le\Xi_0(\mathbf x)\) and, for \(\theta,\theta'\in\Theta\),
\(|\ell(\theta;\Phi(\mathbf x))
-\ell(\theta';\Phi(\mathbf x))|
\le\Xi_\ell(\mathbf x)\|\theta-\theta'\|\).
For \(\theta,\theta'\in\Theta_0\),
\(\sup_{\theta\in\Theta_0}\|\psi(\theta;\Phi(\mathbf x))\|
\le\Xi_1(\mathbf x)\),
\(\sup_{\theta\in\Theta_0}
\|\nabla_\theta\psi(\theta;\Phi(\mathbf x))\|
\le\Xi_2(\mathbf x)\), and\\
\(
\|\nabla_\theta\psi(\theta;\Phi(\mathbf x))
-\nabla_\theta\psi(\theta';\Phi(\mathbf x))\|
\le\Xi_H(\mathbf x)\|\theta-\theta'\|.
\)}

\Litem{Local curvature}{\label{L4}%
\(V:=\nabla^2R(\theta_0)
=\mathbb E\,\nabla_\theta\psi(\theta_0;Z_0)\) is symmetric positive
definite.}%%%%%%%%%%%%%%%%%

%\section{Finite-lag reduction and weighted learning score}

% =====================================================================
% Section 3. Finite-lag reduction for the weighted learning score
%
% The bibliographic key MoldavskayaWeightedLIL is temporary and should
% be replaced by the key used for the companion weighted-LIL paper.
% =====================================================================

\section{Finite-lag reduction for the weighted learning score}
\label{sec:score-reduction}

This section connects the finite-lag learning model of
Section~\ref{sec:model} with the scalar weighted LIL established in the
companion paper.  The main point is that a fixed lag window changes the
leading score only through the deterministic low-frequency coefficient \(\mathcal J\).  All remaining lag effects and all higher Wiener
chaoses are negligible on the iterated-logarithm scale.

All results imported from the companion paper concern the scalar
weighted Hermite reference statistic and its limiting kernels. The
finite-lag reduction and all learning consequences are proved here.

\subsection{Reference chaos and its variance scale}
\label{subsec:reference-chaos}

For \(n\ge r+1\), define the weighted learning-score sum
\(S_n:=\sum_{t=r+1}^{n}a_tg(\mathbf X_t^{(r)})\in\mathbb R^p\).
We use the scalar reference statistic
\(T_n:=\sum_{t=1}^{n}a_tH_m(X_t)\), with exact variance
\[
B_n^2:=\operatorname{Var}(T_n)
=
m!\sum_{s,t=1}^{n}a_sa_t\rho(t-s)^m .
\]
For \(n\ge3\), put \(L_n:=2\log\log n\).
Under \textup{(P\ref{P1})--(P\ref{P3})}, the variance asymptotics
from the companion paper, applied to \(G=H_m\), give
\begin{equation}
\label{eq:reference-variance}
B_n^2
\sim
m!c^m I_{m,\alpha,\kappa}\,
n^{2-2\kappa-\alpha m}
L_a(n)^2L_0(n)^m ,
\end{equation}
where
\(I_{m,\alpha,\kappa}:=
\int_0^1\!\int_0^1
x^{-\kappa}y^{-\kappa}|x-y|^{-\alpha m}\,dx\,dy
=
2\mathrm B(1-\kappa,1-\alpha m)/(2-2\kappa-\alpha m)\).
Consequently, \(B_n\) is regularly varying with index
\(H:=1-\kappa-\alpha m/2>1/2\).

\subsection{Covariance structure of the finite-lag score}
\label{subsec:finite-lag-covariance}

Let \(\Pi_q\) denote coordinatewise orthogonal projection onto the
\(q\)th Wiener chaos. Write
\(g_t:=g(\mathbf X_t^{(r)})\), \(g_t^{[m]}:=\Pi_m g_t\), and define
\(d_t:=g_t^{[m]}-\mathcal J H_m(X_t)\) and
\(u_t:=g_t-g_t^{[m]}\). Thus
\(g_t=\mathcal J H_m(X_t)+d_t+u_t\).
The process \(\{d_t\}\) belongs coordinatewise to the fixed
\(m\)th Wiener chaos, whereas \(\{u_t\}\) contains only chaoses of
orders at least \(m+1\).

\begin{lemma}[Covariance cancellation]
\label{lem:finite-lag-covariance}
Assume \textup{(P\ref{P1})} and \textup{(P\ref{P2})}. Then, as
\(h\to\infty\),
\begin{equation}
\label{eq:leading-lag-covariance}
\operatorname{Cov}\bigl(g_0^{[m]},g_h^{[m]}\bigr)
=
m!\,\mathcal J\mathcal J^\top\rho(h)^m
+
o\bigl(|\rho(h)|^m\bigr),
\end{equation}
where the remainder is understood entrywise. In particular,
\begin{equation}
\label{eq:d-covariance}
\bigl\|\operatorname{Cov}(d_0,d_h)\bigr\|_{\mathrm{op}}
=
o\bigl(|\rho(h)|^m\bigr).
\end{equation}
There is also \(C<\infty\) such that, for all sufficiently large \(h\),
\begin{equation}
\label{eq:u-covariance}
\bigl\|\operatorname{Cov}(u_0,u_h)\bigr\|_{\mathrm{op}}
\le
C\max_{|v|\le r}|\rho(h+v)|^{m+1}
=
O\bigl(|\rho(h)|^{m+1}\bigr).
\end{equation}
\end{lemma}

\begin{proof}
For coordinates \(k,\ell\), the representation in
\eqref{eq:finite-lag-chaos} and the Wiener--It\^o isometry give
\[
\operatorname{Cov}\bigl(g_{0,k}^{[m]},g_{h,\ell}^{[m]}\bigr)
=
\sum_{\mathbf i,\mathbf j}
\beta_{\mathbf i,k}^{(m)}
\beta_{\mathbf j,\ell}^{(m)}
\sum_{\pi\in\mathfrak S_m}
\prod_{\nu=1}^{m}
\rho\bigl(h+i_\nu-j_{\pi(\nu)}\bigr).
\]
Since all lag indices belong to the fixed set
\(\{0,\ldots,r\}\), \textup{(P\ref{P1})} gives
\(\rho(h+v)/\rho(h)\to1\) uniformly over the finitely many relevant
integers \(v\). Hence
\[
\frac{
\operatorname{Cov}\bigl(g_{0,k}^{[m]},g_{h,\ell}^{[m]}\bigr)}
{\rho(h)^m}
\longrightarrow
m!\,\mathcal J_k\mathcal J_\ell,
\]
which proves the first assertion.

The cross-covariance calculation in Section~\ref{Loss gradient},
together with stationarity, gives
\[
\operatorname{Cov}\bigl(g_0^{[m]},H_m(X_h)\bigr)
=
m!\mathcal J\rho(h)^m+o\bigl(|\rho(h)|^m\bigr)
\]
and
\[
\operatorname{Cov}\bigl(H_m(X_0),g_h^{[m]}\bigr)
=
m!\mathcal J^\top\rho(h)^m+o\bigl(|\rho(h)|^m\bigr).
\]
Since
\(\operatorname{Cov}(H_m(X_0),H_m(X_h))=m!\rho(h)^m\),
expanding \(d_t=g_t^{[m]}-\mathcal JH_m(X_t)\) shows that the
leading terms cancel, proving the second assertion.

It remains to control the higher chaoses. Let
\(E_0:=\operatorname{span}\{h_0,\ldots,h_{-r}\}\) and
\(E_h:=\operatorname{span}\{h_h,\ldots,h_{h-r}\}\).
Choose an orthonormal basis \(e_1,\ldots,e_D\) of \(E_0\), writing
each \(e_a\) as a fixed linear combination of
\(h_0,\ldots,h_{-r}\). Since the Gram matrix of these generating
vectors is \(\rho(i-j)\), the same coefficients, with \(h_{-j}\)
replaced by \(h_{h-j}\), define an orthonormal basis
\(e_{1,h},\ldots,e_{D,h}\) of \(E_h\), also when the generating
vectors are linearly dependent. Each
\(\langle e_a,e_{b,h}\rangle\) is therefore a fixed linear
combination of \(\rho(h+v)\), \(|v|\le r\). Hence
\[
\vartheta_h
:=
\|P_{E_0}P_{E_h}\|_{\mathrm{op}}
\le
D\max_{a,b}|\langle e_a,e_{b,h}\rangle|
\le
C\max_{|v|\le r}|\rho(h+v)|
=
O(|\rho(h)|).
\]
Moreover, \(\vartheta_h\le1\), since orthogonal projections are
contractions.

For \(F\in E_0^{\odot q}\) and \(G\in E_h^{\odot q}\),
\(G=P_{E_h}^{\otimes q}G\), and therefore
\(
\langle F,G\rangle
=
\left\langle
F,(P_{E_0}P_{E_h})^{\otimes q}G
\right\rangle.
\)
Since
\(\|A^{\otimes q}\|_{\mathrm{op}}=\|A\|_{\mathrm{op}}^q\),
\(
|\langle F,G\rangle|
\le
\vartheta_h^q\|F\|\,\|G\|.
\)
Using the chaos expansions of \(u_{0,k}\) and \(u_{h,\ell}\),
orthogonality of different chaos orders, \(\vartheta_h\le1\), and
Cauchy--Schwarz gives
\[
|\operatorname{Cov}(u_{0,k},u_{h,\ell})|
\le
\vartheta_h^{m+1}
\|u_{0,k}\|_2\|u_{h,\ell}\|_2.
\]
Stationarity and the preceding estimate for \(\vartheta_h\) prove the
third assertion.
\end{proof}
%%%%%%%%%%%%%%%%%%%
\subsection{Block variance bounds}
\label{subsec:block-variance}

The covariance cancellation in
Lemma~\ref{lem:finite-lag-covariance} is sufficient for the
almost-sure reduction below. No differentiability assumption on
\(L_a\) or spectral assumption is required at this stage.

\begin{lemma}[Variance gain for the lag and higher-chaos remainders]
\label{lem:block-variance}
Assume \textup{(P\ref{P1})--(P\ref{P3})}. There exist constants
\(q>1\), \(\delta>0\), and a deterministic sequence
\(\varepsilon_N\downarrow0\) such that, for every sufficiently large
\(N\) and every interval \(I\subset\{N+1,\ldots,2N\}\),
\begin{equation}
\label{eq:d-block}
\mathbb E
\left\|
    \sum_{t\in I}a_td_t
\right\|^2
\le
\varepsilon_N B_N^2
\left(\frac{|I|}{N}\right)^q
\end{equation}
and
\begin{equation}
\label{eq:u-block}
\mathbb E
\left\|
    \sum_{t\in I}a_tu_t
\right\|^2
\le
C B_N^2N^{-2\delta}
\left(\frac{|I|}{N}\right).
\end{equation}
Moreover,
\begin{equation}
\label{eq:d-global}
\mathbb E
\left\|
    \sum_{t=r+1}^{n}a_td_t
\right\|^2
=
o(B_n^2)
\end{equation}
and
\begin{equation}
\label{eq:u-global}
\mathbb E
\left\|
    \sum_{t=r+1}^{n}a_tu_t
\right\|^2
\le
C B_n^2n^{-2\delta}.
\end{equation}
\end{lemma}

\begin{proof}
Put \(\beta:=\alpha m\in(0,1)\). Fix
\(\eta\in(0,1-\beta)\) and set
\(q:=2-\beta-\eta>1\). By regular variation, uniformly for
\(N\le t\le2N\),
\(
    |a_t|
    \le
    C N^{-\kappa}L_a(N).
\)
Potter's bound and the standard summation estimate for a regularly
varying kernel give, uniformly over all intervals
\(I\subset\{N+1,\ldots,2N\}\),
\begin{equation}
\label{eq:reference-block-convolution}
\sum_{s,t\in I}
|a_sa_t|\,
(1+|t-s|)^{-\beta}
L_0(1+|t-s|)^m
\le
C B_N^2
\left(\frac{|I|}{N}\right)^q.
\end{equation}

By \eqref{eq:d-covariance},
\(
    \bigl\|\operatorname{Cov}(d_0,d_h)\bigr\|_{\mathrm{op}}
    =
    o\bigl(h^{-\beta}L_0(h)^m\bigr).
\)
Define
\[
    \omega(M)
    :=
    \sup_{h\ge M}
    \frac{
        \bigl\|\operatorname{Cov}(d_0,d_h)\bigr\|_{\mathrm{op}}
    }{
        h^{-\beta}L_0(h)^m
    }.
\]
Then \(\omega(M)\to0\) as \(M\to\infty\). Since \(L_0\) is slowly
varying, Potter's bound gives
\(L_0(N)^m\ge cN^{-\eta/2}\) for all sufficiently large \(N\);
hence
\(
    N^\eta L_0(N)^m\longrightarrow\infty.
\)
We may therefore choose integers \(M_N\uparrow\infty\) sufficiently
slowly that
\(
   M_N/(N^\eta L_0(N)^m)
    \longrightarrow0.
\)

Using stationarity and summing over the fixed number of coordinates,
split the covariance sum for
\(\mathbb E\|\sum_{t\in I}a_td_t\|^2\) according as
\(|t-s|>M_N\) or \(|t-s|\le M_N\). On the first part,
\[
    \bigl\|\operatorname{Cov}(d_s,d_t)\bigr\|_{\mathrm{op}}
    \le
    \omega(M_N)
    |t-s|^{-\beta}L_0(|t-s|)^m,
\]
so \eqref{eq:reference-block-convolution} bounds its contribution by
\(
    C\omega(M_N)B_N^2
    \left(|I|/N\right)^q.
\)

For \(|t-s|\le M_N\), square integrability and stationarity give a
uniform bound on the covariance, while the number of relevant pairs
is at most \(C|I|M_N\). Consequently, this contribution is at most\\
\(
    C N^{-2\kappa}L_a(N)^2|I|M_N.
\)
By \eqref{eq:reference-variance},
\(
    B_N^2
    \asymp
    N^{2-2\kappa-\beta}
    L_a(N)^2L_0(N)^m.
\)
Dividing the near-lag bound by
\(B_N^2(|I|/N)^q\), and using
\(q=2-\beta-\eta\), gives
\[
    \frac{M_N}{N^\eta L_0(N)^m}|I|^{1-q}
    \le
    \frac{M_N}{N^\eta L_0(N)^m},
\]
because \(q>1\) and \(|I|\ge1\). Thus
\[
    \mathbb E
    \left\|
        \sum_{t\in I}a_td_t
    \right\|^2
    \le
    \widetilde\varepsilon_N B_N^2
    \left(\frac{|I|}{N}\right)^q,
\]
where
\[
    \widetilde\varepsilon_N
    :=
    C\left\{
        \omega(M_N)
        +
        \frac{M_N}{N^\eta L_0(N)^m}
    \right\}
    \longrightarrow0.
\]
Replacing \(\widetilde\varepsilon_N\), if necessary, by its decreasing
majorant
\(\varepsilon_N:=\sup_{K\ge N}\widetilde\varepsilon_K\)
proves \eqref{eq:d-block} with
\(\varepsilon_N\downarrow0\).

We next prove \eqref{eq:d-global}. Fix \(M\). 
For \(|t-s|>M\), Lemma~10.1 of
\cite{Moldavskaya2026WeightedLIL}, applied with
\(\beta=\alpha m\), gives together with the definition of
\(\omega(M)\)
\[
    \sum_{\substack{r+1\le s,t\le n\\|t-s|>M}}
    |a_sa_t|\,
    \bigl\|\operatorname{Cov}(d_s,d_t)\bigr\|_{\mathrm{op}}
    \le
    C\omega(M)B_n^2.
\]
For \(|t-s|\le M\), regular variation of the weights at each fixed lag
gives
\[
    \sum_{\substack{r+1\le s,t\le n\\|t-s|\le M}}
    |a_sa_t|
    \le
    C_M\sum_{t\le n}a_t^2.
\]
Since \(2\kappa<1\), Karamata's theorem yields
\[
    \sum_{t\le n}a_t^2
    =
    O\bigl(n^{1-2\kappa}L_a(n)^2\bigr)
    =
    o(B_n^2).
\]
Indeed, by \eqref{eq:reference-variance}, the ratio is of order
\(n^{\beta-1}L_0(n)^{-m}\to0\) by Potter's bound, since
\(\beta<1\). It follows that
\[
    \limsup_{n\to\infty}
    \frac{
        \mathbb E
        \left\|
            \sum_{t=r+1}^{n}a_td_t
        \right\|^2
    }{
        B_n^2
    }
    \le
    C\omega(M).
\]
Letting \(M\to\infty\) proves \eqref{eq:d-global}.

For the higher-chaos remainder,
Lemma~\ref{lem:finite-lag-covariance} gives
\[
    \bigl\|\operatorname{Cov}(u_0,u_h)\bigr\|_{\mathrm{op}}
    \le
    C|\rho(h)|^{m+1}
\]
for all sufficiently large \(h\). Each coordinate of \(\{u_t\}\) is
stationary, centered, and square integrable. The proof of the
higher-rank variance-gain estimate
\cite[Lemma~10.2]{Moldavskaya2026WeightedLIL}
uses, after establishing precisely such a covariance bound, only the
regular variation of the deterministic weights. Applying that
argument to each coordinate and summing over the fixed number of
coordinates yields constants \(\delta>0\) and \(\gamma\ge1\) such
that
\[
    \mathbb E
    \left\|
        \sum_{t\in I}a_tu_t
    \right\|^2
    \le
    C B_N^2N^{-2\delta}
    \left(\frac{|I|}{N}\right)^\gamma
\]
for every sufficiently large \(N\) and every
\(I\subset\{N+1,\ldots,2N\}\), and
\[
    \mathbb E
    \left\|
        \sum_{t=r+1}^{n}a_tu_t
    \right\|^2
    \le
    C B_n^2n^{-2\delta}.
\]
Here the fixed omission of the first \(r\) terms is harmless; after
decreasing \(\delta\), if necessary, the corresponding finite
\(L^2\)-contribution is absorbed by the right-hand side.

Since \(|I|/N\le1\) and \(\gamma\ge1\),
\(
    \left(|I|/N\right)^\gamma
    \le
    |I|/N.
\)
This proves \eqref{eq:u-block} and \eqref{eq:u-global}.
\end{proof}

\subsection{Almost-sure finite-lag reduction}
\label{subsec:as-reduction}

We now prove the bridge result used throughout the remainder of the
paper.

\begin{theorem}[Finite-lag reduction of the weighted learning score]
\label{thm:finite-lag-reduction}
Assume \textup{(P\ref{P1})--(P\ref{P3})}. Then
\begin{equation}
\label{eq:score-reduction}
\frac{\|S_n-\mathcal JT_n\|}{B_nL_n^{m/2}}
\longrightarrow0
\qquad\text{almost surely}.
\end{equation}
Equivalently,
\begin{equation}
\label{eq:score-leading-representation}
S_n
=
\mathcal J\sum_{t=1}^{n}a_tH_m(X_t)
+
o_{\mathrm{a.s.}}\bigl(B_nL_n^{m/2}\bigr).
\end{equation}
\end{theorem}

\begin{proof}
Put \(D_n:=\sum_{t=r+1}^{n}a_td_t\) and
\(U_n:=\sum_{t=r+1}^{n}a_tu_t\). Then
\(S_n=\mathcal J\sum_{t=r+1}^{n}a_tH_m(X_t)+D_n+U_n\).
The fixed difference between the first sum and \(T_n\) is negligible
on the scale \(B_nL_n^{m/2}\). All scalar tail and maximal estimates
below are applied coordinatewise; since \(p\) is fixed, this changes
only the constants.

Let \(n_k=2^k\). By \eqref{eq:d-global},
\(\eta_n:=B_n^{-2}\mathbb E\|D_n\|^2\to0\).
Hypercontractivity in the fixed \(m\)th chaos gives, for every \(x>0\),
\[
\mathbb P\!\left(
\|D_{n_k}\|>xB_{n_k}L_{n_k}^{m/2}\right)
\le
C\exp\{-c x^{2/m}\eta_{n_k}^{-1/m}L_{n_k}\},
\]
which is summable since \(L_{n_k}\asymp\log k\) and
\(\eta_{n_k}^{-1/m}\to\infty\).
By \eqref{eq:d-block} and
\cite[Lemma~10.5]{Moldavskaya2026WeightedLIL}, the same argument
controls \(\max_{n_k\le n\le n_{k+1}}\|D_n-D_{n_k}\|\), with
\(\eta_{n_k}\) replaced by \(\varepsilon_{n_k}\).
Hence Borel--Cantelli and regular variation of \(B_n\) give
\(\|D_n\|/(B_nL_n^{m/2})\to0\) almost surely.

For \(U_n\), \eqref{eq:u-global} and Chebyshev's inequality give
summable bounds along \(n_k=2^k\). Moreover,
\eqref{eq:u-block} and
\cite[Lemma~10.3]{Moldavskaya2026WeightedLIL} yield
\[
\mathbb E\max_{n_k\le n\le n_{k+1}}
\|U_n-U_{n_k}\|^2
\le
C(\log n_k)^2B_{n_k}^2n_k^{-2\delta}.
\]
Chebyshev's inequality and Borel--Cantelli therefore give
\(\|U_n\|/(B_nL_n^{m/2})\to0\) almost surely.
Combining the two remainders proves \eqref{eq:score-reduction}.
\end{proof}

\subsection{Uniform reduction of the score trajectory}
\label{subsec:uniform-score-reduction}

For integers \(k\le r\), set \(S_k=0\). For \(0\le t\le T\), define
\(S_n(t):=(1-\{nt\})S_{\lfloor nt\rfloor}
+\{nt\}S_{\lfloor nt\rfloor+1}\), and define \(T_n(t)\) analogously.

\begin{corollary}[Uniform finite-lag reduction]
\label{cor:uniform-score-reduction}
Under the assumptions of Theorem~\ref{thm:finite-lag-reduction}, for
every fixed \(T>0\),
\begin{equation}
\label{eq:uniform-score-reduction}
\sup_{0\le t\le T}
\frac{\|S_n(t)-\mathcal JT_n(t)\|}
{B_nL_n^{m/2}}
\longrightarrow0
\qquad\text{almost surely}.
\end{equation}
\end{corollary}

\begin{proof}
Put \(E_k:=S_k-\mathcal JT_k\). By
Theorem~\ref{thm:finite-lag-reduction},
\(\|E_k\|/(B_kL_k^{m/2})\to0\) a.s. The sequence
\(B_kL_k^{m/2}\) is regularly varying with positive index \(H\).
Therefore, for every fixed \(T\), Potter's bound gives
\(\sup_{3\le k\le Tn+1}B_kL_k^{m/2}/(B_nL_n^{m/2})\le C_T\)
for all sufficiently large \(n\). The finitely many initial values are
negligible, and hence
\(\max_{0\le k\le Tn+1}\|E_k\|/(B_nL_n^{m/2})\to0\) a.s.
Linear interpolation does not change this conclusion.
\end{proof}

\subsection{Consequences for the weighted learning score}
\label{subsec:score-consequences}

The reduction theorem first yields an asymptotic covariance statement.

\begin{corollary}[Asymptotic covariance and one-dimensional geometry]
\label{cor:score-covariance}
Under the assumptions of
Theorem~\ref{thm:finite-lag-reduction},
\begin{equation}
\label{eq:score-covariance}
B_n^{-2}\operatorname{Cov}(S_n)
\longrightarrow
\mathcal J\mathcal J^\top .
\end{equation}
Consequently, for every \(v\in\mathbb R^p\) with
\(v^\top\mathcal J\neq0\),
\(\operatorname{Var}(v^\top S_n)\sim
(v^\top\mathcal J)^2B_n^2\), whereas
\(\operatorname{Var}(v^\top S_n)=o(B_n^2)\) whenever
\(v^\top\mathcal J=0\).
\end{corollary}
\begin{proof}
Lemma~\ref{lem:block-variance}, together with the fixed difference
between the sums starting at \(1\) and at \(r+1\), which is bounded in
\(L^2\) and hence \(o(B_n^2)\) because \(B_n\to\infty\), gives
\(
    \mathbb E
    \left\|
       S_n-\mathcal JT_n
    \right\|^2
    =
    o(B_n^2).
\)
The result follows from \(\operatorname{Var}(T_n)=B_n^2\) and Cauchy--Schwarz for the cross-covariance terms.
\end{proof}

For \(m\ge2\), let \(Q_t^{(m,\alpha,\kappa)}\) be the
variance-normalized limiting kernel of the companion paper
\cite{Moldavskaya2026WeightedLIL}, with weight exponent \(\kappa\),
and let \(\mathfrak H_{\mathbb R}\) denote the real spectral Hilbert
space used there. Define
\[
\mathcal K_{m,\alpha,\kappa;T}
:=
\left\{
t\mapsto
\left\langle Q_t^{(m,\alpha,\kappa)},\xi^{\otimes m}\right\rangle
:
\|\xi\|_{\mathfrak H_{\mathbb R}}\le1
\right\}.
\]
Set
\(\Lambda_{m,\alpha,\kappa}:=
\sup_{\|\xi\|_{\mathfrak H_{\mathbb R}}\le1}
|\langle Q_1^{(m,\alpha,\kappa)},\xi^{\otimes m}\rangle|\);
for \(m=1\), set \(\Lambda_{1,\alpha,\kappa}:=1\).

\begin{corollary}[Exact weighted LIL for the learning score]
\label{cor:score-LIL}
Assume \textup{(P\ref{P1})--(P\ref{P3})}.

If \(m=1\), then
\begin{equation}
\label{eq:score-LIL-linear}
\limsup_{n\to\infty}
\frac{\|S_n\|}{B_nL_n^{1/2}}
=
\|\mathcal J\|
\qquad\text{a.s.}
\end{equation}

If \(m\ge2\) and \textup{(P\ref{P4})--(P\ref{P6})} also hold, then
\begin{equation}
\label{eq:score-LIL-nonlinear}
\limsup_{n\to\infty}
\frac{\|S_n\|}{B_nL_n^{m/2}}
=
\|\mathcal J\|\Lambda_{m,\alpha,\kappa}
\qquad\text{a.s.}
\end{equation}
Moreover, for every fixed \(T>0\), the normalized score trajectories
are almost surely relatively compact in \(C([0,T];\mathbb R^p)\), and
\begin{equation}
\label{eq:score-functional-cluster}
\operatorname{Cl}_{C([0,T];\mathbb R^p)}
\left(
\frac{S_n(\cdot)}{B_nL_n^{m/2}}
\right)
=
\mathcal J\mathcal K_{m,\alpha,\kappa;T}
\qquad\text{a.s.},
\end{equation}
where
\(\mathcal J\mathcal K_{m,\alpha,\kappa;T}
:=\{t\mapsto\mathcal J\varphi(t):
\varphi\in\mathcal K_{m,\alpha,\kappa;T}\}\).
\end{corollary}
\begin{proof}
For \(m=1\), apply the sharp linear weighted LIL of the companion
paper to \(T_n\) and use
Theorem~\ref{thm:finite-lag-reduction}.
For \(m\ge2\), apply its nonlinear functional cluster theorem to
\(T_n(\cdot)\). Corollary~\ref{cor:uniform-score-reduction} shows that
the normalized score trajectories have, after multiplication by
\(\mathcal J\), the same relative compactness and cluster set.
Evaluation at \(t=1\) gives \eqref{eq:score-LIL-nonlinear}.
\end{proof}

\begin{remark}[Pathwise geometry of the learning score]
\label{rem:score-rank-one-geometry}
Equations \eqref{eq:score-covariance} and
\eqref{eq:score-functional-cluster} show that the leading \(p\)-dimensional score fluctuations are asymptotically confined to
the deterministic direction \(\mathcal J\).  This is not an
assumption of a one-dimensional parameter model.  It is a consequence
of scalar Gaussian long-memory input together with finite-lag
subordination: after low-frequency reduction, all lag coordinates
share the same leading long-memory component.
\end{remark}

% =====================================================================
% Section 4. Exact LIL for the learning error
% =====================================================================

\section{Exact LIL for the learning error}
\label{sec:parameter-LIL}

We transfer the weighted score LIL of Section~\ref{sec:score-reduction}
to the learned parameter. Put
\(\Delta_n:=\widehat\theta_n-\theta_0\) and
\(\widehat V_n(\theta):=
W_n^{-1}\sum_{t=r+1}^{n}
a_t\nabla_\theta\psi(\theta;Z_t)\).
The following standard weighted-ERM facts are collected for later use;
details are given in the Supplement
\citep{Moldavskaya2026MLSupplement}.

\begin{proposition}[Uniform laws, consistency, and local curvature]
\label{prop:local-Hessian}
Assume \textup{(P\ref{P1})}, \textup{(P\ref{P3})}, and
\textup{(L\ref{L1})--(L\ref{L4})}. Then, almost surely,
\begin{equation}
\label{eq:risk-uniform-law}
\sup_{\theta\in\Theta}
|\widehat R_n(\theta)-R(\theta)|
\longrightarrow0,
\qquad
\widehat\theta_n\longrightarrow\theta_0.
\end{equation}
Moreover, for every compact \(K\subset\Theta_0\),
\begin{equation}
\label{eq:hessian-uniform-law}
\sup_{\theta\in K}
\|\widehat V_n(\theta)-V(\theta)\|
\longrightarrow0,
\qquad
V(\theta):=
\mathbb E\nabla_\theta\psi(\theta;Z_0).
\end{equation}
For all sufficiently large \(n\), define
\begin{equation}
\label{eq:integrated-Hessian}
\overline V_n
:=
\int_0^1
\widehat V_n(\theta_0+u\Delta_n)\,du.
\end{equation}
Then
\begin{equation}
\label{eq:integrated-Hessian-convergence}
\overline V_n\longrightarrow V:=V(\theta_0),
\qquad
\overline V_n^{-1}\longrightarrow V^{-1},
\end{equation}
and
\begin{equation}
\label{eq:exact-score-identity}
\Delta_n
=
-\,W_n^{-1}\overline V_n^{-1}S_n.
\end{equation}
\end{proposition}

\begin{proof}
Since \(\rho(h)\to0\), the Gaussian sequence and its finite-lag
measurable factors are ergodic. A weighted ergodic theorem for the
regularly varying weights, together with a finite-net argument and the
\(L^2\)-envelopes in \textup{(L\ref{L3})}, yields
\eqref{eq:risk-uniform-law}--\eqref{eq:hessian-uniform-law};
identifiability gives consistency. Uniform Hessian convergence along
the segment between \(\theta_0\) and \(\widehat\theta_n\) yields
\eqref{eq:integrated-Hessian-convergence}. Since
\(\theta_0\in\operatorname{int}\Theta\) and
\(\widehat\theta_n\to\theta_0\), almost surely
\(\widehat\theta_n\in\operatorname{int}\Theta\) for all sufficiently
large \(n\), so
\(\nabla\widehat R_n(\widehat\theta_n)=0\).
The integral Taylor formula then gives
\eqref{eq:exact-score-identity}. Details are given in the Supplement.
\end{proof}

Define the natural learning-error scale
\(r_n:=B_nL_n^{m/2}/W_n\).

\begin{theorem}[LIL-scale Bahadur representation]
\label{thm:LIL-Bahadur}
Assume \textup{(P\ref{P1})--(P\ref{P3})} and
\textup{(L\ref{L1})--(L\ref{L4})}. Then
\begin{equation}
\widehat\theta_n-\theta_0
=
-\,W_n^{-1}V^{-1}S_n+o_{\mathrm{a.s.}}(r_n).
\end{equation}
Consequently,
\begin{equation}
\widehat\theta_n-\theta_0
=
-\,W_n^{-1}V^{-1}\mathcal J
\sum_{t=1}^{n}a_tH_m(X_t)
+
o_{\mathrm{a.s.}}(r_n),
\end{equation}
and
\(\|\widehat\theta_n-\theta_0\|=O_{\mathrm{a.s.}}(r_n)\).
\end{theorem}

\begin{proof}
Theorem~\ref{thm:finite-lag-reduction} and the general upper bound
\cite[Theorem~3.2]{Moldavskaya2026WeightedLIL}, applied to \(T_n\),
give \(\|S_n\|=O_{\mathrm{a.s.}}(B_nL_n^{m/2})\).
Hence \eqref{eq:exact-score-identity} and
\(\overline V_n^{-1}\to V^{-1}\) imply
\[
\Delta_n+W_n^{-1}V^{-1}S_n
=
-\,W_n^{-1}(\overline V_n^{-1}-V^{-1})S_n
=
o_{\mathrm{a.s.}}(r_n).
\]

The second assertion follows from
Theorem~\ref{thm:finite-lag-reduction}, and the last from the first
and the preceding score bound.
\end{proof}

Put \(q_0:=V^{-1}\mathcal J\in\mathbb R^p\).

\begin{theorem}[Exact LIL and cluster geometry of the learning error]
\label{thm:parameter-LIL}
Under the assumptions of Theorem~\ref{thm:LIL-Bahadur}; if \(m\ge2\),
assume also \textup{(P\ref{P4})--(P\ref{P6})}.
\begin{equation}
\label{eq:parameter-limsup}
\limsup_{n\to\infty}
\frac{W_n\|\widehat\theta_n-\theta_0\|}
     {B_nL_n^{m/2}}
=
\|q_0\|\Lambda_{m,\alpha,\kappa}
\qquad\text{a.s.}
\end{equation}
where \(\Lambda_{1,\alpha,\kappa}=1\). More generally, for every
\(v\in\mathbb R^p\),
\begin{equation}
\label{eq:directional-parameter-LIL}
\limsup_{n\to\infty}
\frac{W_n|v^\top(\widehat\theta_n-\theta_0)|}
     {B_nL_n^{m/2}}
=
|v^\top q_0|\Lambda_{m,\alpha,\kappa}
\qquad\text{a.s.}
\end{equation}
If \(v^\top q_0=0\), then
\begin{equation}
\label{eq:orthogonal-directions}
\frac{W_n v^\top(\widehat\theta_n-\theta_0)}
     {B_nL_n^{m/2}}
\longrightarrow0
\qquad\text{a.s.}
\end{equation}

For \(m\ge2\), let
\[
\mathcal C_{m,\alpha,\kappa}
:=
\left\{
\left\langle Q_1^{(m,\alpha,\kappa)},\xi^{\otimes m}\right\rangle:
\|\xi\|_{\mathfrak H_{\mathbb R}}\le1
\right\}.
\]
Then
\begin{equation}
\label{eq:parameter-cluster}
\operatorname{Clust}_{n\to\infty}
\left\{
\frac{W_n(\widehat\theta_n-\theta_0)}
     {B_nL_n^{m/2}}
\right\}
=
-q_0\,\mathcal C_{m,\alpha,\kappa}
\qquad\text{a.s.}
\end{equation}
\end{theorem}

\begin{proof}
By Theorem~\ref{thm:LIL-Bahadur},
\begin{equation}
\label{eq:normalized-parameter-representation}
\frac{W_n(\widehat\theta_n-\theta_0)}
     {B_nL_n^{m/2}}
=
-q_0\frac{T_n}{B_nL_n^{m/2}}
+o_{\mathrm{a.s.}}(1).
\end{equation}
The assertions follow from the sharp weighted LIL and, for \(m\ge2\),
the endpoint cluster theorem for \(T_n\).
\end{proof}

Thus the leading learning error lies in the direction
\(V^{-1}\mathcal J\), while orthogonal directions are negligible on
the LIL scale; Section~\ref{sec:rates} shows that \(r_n\) has
polynomial exponent \(-\alpha m/2\), independent of \(\kappa\).

\begin{remark}[Parity of the endpoint cluster set]
\label{rem:parameter-cluster-parity}
For the reference statistic \(T_n\), the explicit kernel representation
in the companion paper gives
\(\langle Q_1^{(m,\alpha,\kappa)},\xi^{\otimes m}\rangle
=C\int_0^1y^{-\kappa}\phi_\xi(y)^m\,dy\), with \(C>0\) and
real \(\phi_\xi\). Hence
\(\mathcal C_{m,\alpha,\kappa}=[0,\Lambda_{m,\alpha,\kappa}]\)
for even \(m\), whereas
\(\mathcal C_{m,\alpha,\kappa}
=[-\Lambda_{m,\alpha,\kappa},\Lambda_{m,\alpha,\kappa}]\)
for odd \(m\). Thus the leading parameter-error cluster set is a
one-sided segment for even \(m\) and a symmetric segment for odd \(m\).
\end{remark}
%\section{Functional LIL for the learning trajectory}
% =====================================================================
% Section 5. Functional LIL for the learning trajectory
% =====================================================================

\section{Functional LIL for the learning trajectory}
\label{sec:functional-learning}

We now describe the complete learning trajectory. Throughout this
section \(m\ge2\), so that the nonlinear functional cluster theorem of
the companion weighted-LIL paper
\cite{Moldavskaya2026WeightedLIL} is available.

For \(k\ge r+1\), put
\(\Gamma_k:=W_k(\widehat\theta_k-\theta_0)\), and set
\(\Gamma_k=0\) for \(0\le k\le r\). For \(T>0\), let
\(\mathcal E_n\) be the linear interpolation on \([0,T]\) of
\(\mathcal E_n(k/n)=\Gamma_k\).
Recall \(q_0=V^{-1}\mathcal J\), and define
\begin{equation}
\label{eq:parameter-functional-cluster-set}
\mathcal K^{\theta}_{m,\alpha,\kappa;T}
:=
\left\{
t\mapsto
-q_0
\left\langle Q_t^{(m,\alpha,\kappa)},\xi^{\otimes m}\right\rangle
:
\|\xi\|_{\mathfrak H_{\mathbb R}}\le1
\right\}.
\end{equation}

\begin{theorem}[Functional LIL for the weighted learning trajectory]
\label{thm:functional-parameter-LIL}
Assume \textup{(P\ref{P1})--(P\ref{P6})} and
\textup{(L\ref{L1})--(L\ref{L4})}. Then, for every fixed \(T>0\),
the sequence
\(\{\mathcal E_n(\cdot)/(B_nL_n^{m/2})\}_{n\ge3}\)
is almost surely relatively compact in
\(C([0,T];\mathbb R^p)\), and
\begin{equation}
\label{eq:functional-parameter-cluster}
\operatorname{Cl}_{C([0,T];\mathbb R^p)}
\left(
\frac{\mathcal E_n(\cdot)}{B_nL_n^{m/2}}
\right)
=
\mathcal K^{\theta}_{m,\alpha,\kappa;T}
\qquad\text{a.s.}
\end{equation}
\end{theorem}

\begin{proof}
For all sufficiently large \(k\),
Proposition~\ref{prop:local-Hessian} gives
\(\Gamma_k=-\overline V_k^{-1}S_k\). Hence
\[
\Gamma_k+V^{-1}S_k
=
-(\overline V_k^{-1}-V^{-1})S_k.
\]
Since \(\overline V_k^{-1}\to V^{-1}\) a.s. and the normalized score
trajectories are almost surely relatively compact by
Corollary~\ref{cor:score-LIL}, they are uniformly bounded on
\([0,T]\). Splitting off finitely many initial \(k\)'s therefore gives
\begin{equation}
\label{eq:uniform-Bahadur}
\sup_{0\le t\le T}
\frac{\|\mathcal E_n(t)+V^{-1}S_n(t)\|}
{B_nL_n^{m/2}}
\longrightarrow0
\qquad\text{a.s.}
\end{equation}
Linear interpolation does not change the conclusion.
Relative compactness of
\(\mathcal E_n(\cdot)/(B_nL_n^{m/2})\) follows from
\eqref{eq:uniform-Bahadur} and that of
\(S_n(\cdot)/(B_nL_n^{m/2})\) in
Corollary~\ref{cor:score-LIL}.
The functional score cluster set is
\(\mathcal J\mathcal K_{m,\alpha,\kappa;T}\); hence
\eqref{eq:uniform-Bahadur} and the continuous map
\(f\mapsto-V^{-1}f\) yield
\eqref{eq:functional-parameter-cluster}.
\end{proof}

Thus every leading cluster trajectory takes values in the
one-dimensional subspace
\(\operatorname{span}\{V^{-1}\mathcal J\}\).

\subsection{The normalized learning curve}

Recall \(r_n:=B_nL_n^{m/2}/W_n\). Fix \(0<\varepsilon<T\), and let
\(\mathcal D_n\) be the linear interpolation on \([\varepsilon,T]\)
of
\(\mathcal D_n(k/n):=
(\widehat\theta_k-\theta_0)/r_n\).
Define
\begin{equation}
\label{eq:raw-learning-cluster}
\mathcal G_{m,\alpha,\kappa;\varepsilon,T}
:=
\left\{
t\mapsto
-t^{\kappa-1}q_0
\left\langle Q_t^{(m,\alpha,\kappa)},\xi^{\otimes m}\right\rangle
:
\|\xi\|_{\mathfrak H_{\mathbb R}}\le1
\right\}.
\end{equation}

\begin{theorem}[Functional LIL for the normalized learning curve]
\label{thm:raw-learning-curve}
Under the assumptions of
Theorem~\ref{thm:functional-parameter-LIL}, the sequence
\(\{\mathcal D_n\}\) is almost surely relatively compact in
\(C([\varepsilon,T];\mathbb R^p)\), and
\begin{equation}
\label{eq:raw-learning-curve-cluster}
\operatorname{Cl}_{C([\varepsilon,T];\mathbb R^p)}
(\mathcal D_n)
=
\mathcal G_{m,\alpha,\kappa;\varepsilon,T}
\qquad\text{a.s.}
\end{equation}
\end{theorem}

\begin{proof}
Karamata's theorem and the uniform convergence theorem for regularly
varying functions give
\begin{equation}
\label{eq:uniform-W-ratio}
\sup_{\varepsilon\le t\le T}
\left|
\frac{W_n}{W_{\lfloor nt\rfloor}}-t^{\kappa-1}
\right|
\longrightarrow0.
\end{equation}
At mesh points,
\[
\frac{\widehat\theta_k-\theta_0}{r_n}
=
\frac{W_n}{W_k}
\frac{\Gamma_k}{B_nL_n^{m/2}}.
\]
Let \(c_n\) be the linear interpolation of the mesh values
\(W_n/W_k\). By \eqref{eq:uniform-W-ratio},
\(c_n\to t^{\kappa-1}\) uniformly, while
Theorem~\ref{thm:functional-parameter-LIL} makes the normalized
\(\mathcal E_n\)'s relatively compact and hence uniformly bounded and
equicontinuous. Thus interpolation of the preceding products differs
uniformly by \(o(1)\) from
\(t^{\kappa-1}\mathcal E_n(t)/(B_nL_n^{m/2})\).
The continuous multiplication map
\(f(t)\mapsto t^{\kappa-1}f(t)\) now gives
\eqref{eq:raw-learning-curve-cluster}.
\end{proof}

The restriction to \([\varepsilon,T]\) is intrinsic to this terminal
normalization: \(t^{\kappa-1}\) is singular at zero because
\(\kappa<1\) under \textup{(P\ref{P3})}.

\subsection{Pointwise learning envelope}

\begin{lemma}[Scaling of the limiting kernel]
\label{lem:kernel-time-scaling}
Let \(H:=1-\kappa-\alpha m/2\). For \(t>0\),
\begin{equation}
\label{eq:kernel-time-scaling}
Q_t^{(m,\alpha,\kappa)}
=
t^H\mathcal U_t^{\otimes m}Q_1^{(m,\alpha,\kappa)},
\qquad
(\mathcal U_t h)(x):=t^{1/2}h(tx).
\end{equation}
Consequently,
\begin{equation}
\label{eq:Lambda-time-scaling}
\sup_{\|\xi\|_{\mathfrak H_{\mathbb R}}\le1}
\left|
\left\langle
Q_t^{(m,\alpha,\kappa)},\xi^{\otimes m}
\right\rangle
\right|
=
t^H\Lambda_{m,\alpha,\kappa}.
\end{equation}
\end{lemma}

\begin{proof}
From the explicit kernel representation in
\cite{Moldavskaya2026WeightedLIL}, the substitution \(y=tu\) gives
\(Q_t^{(m,\alpha,\kappa)}
=t^{1-\kappa-\alpha m/2}
\mathcal U_t^{\otimes m}Q_1^{(m,\alpha,\kappa)}\).
Since \(\mathcal U_t\) is unitary on
\(\mathfrak H_{\mathbb R}\), taking the supremum over the unit ball
proves \eqref{eq:Lambda-time-scaling}.
\end{proof}

\begin{corollary}[Exact pointwise learning envelope]
\label{cor:pointwise-learning-envelope}
Under the assumptions of
Theorem~\ref{thm:functional-parameter-LIL}, for every fixed \(t>0\),
\begin{equation}
\label{eq:weighted-pointwise-envelope}
\limsup_{n\to\infty}
\frac{
W_{\lfloor nt\rfloor}
\|\widehat\theta_{\lfloor nt\rfloor}-\theta_0\|
}{
B_nL_n^{m/2}
}
=
t^H\|q_0\|\Lambda_{m,\alpha,\kappa}
\qquad\text{a.s.},
\end{equation}
and
\begin{equation}
\label{eq:raw-pointwise-envelope}
\limsup_{n\to\infty}
\frac{
\|\widehat\theta_{\lfloor nt\rfloor}-\theta_0\|
}{
r_n
}
=
t^{-\alpha m/2}
\|q_0\|\Lambda_{m,\alpha,\kappa}
\qquad\text{a.s.}
\end{equation}
\end{corollary}

\begin{proof}
Evaluation at \(t\) in
Theorem~\ref{thm:functional-parameter-LIL}, together with
Lemma~\ref{lem:kernel-time-scaling}, gives
\eqref{eq:weighted-pointwise-envelope}. Multiplying by
\(W_n/W_{\lfloor nt\rfloor}\to t^{\kappa-1}\) and using
\(H+\kappa-1=-\alpha m/2\) gives
\eqref{eq:raw-pointwise-envelope}.
\end{proof}

Thus the power profile \(t^{-\alpha m/2}\) is independent of
\(\kappa\): weighting changes the sharp constant and the cluster
geometry, but not the polynomial learning profile.

Finally, every continuously differentiable functional
\(F:\Theta_0\to\mathbb R^s\) inherits the same functional LIL through
the pathwise delta method, with effective direction
\(DF(\theta_0)V^{-1}\mathcal J\). In particular, for a fixed feature
vector \(u\), this gives the corresponding cluster set for the
prediction error
\(f_{\widehat\theta_{\lfloor nt\rfloor}}(u)-f_{\theta_0}(u)\).
The detailed delta-method statement and proof are given in the
Supplement.
%\section{Learning rates, excess risk, and optimal weighting}
% =====================================================================
% Section 6. Learning rates, excess risk, and optimal weighting
% =====================================================================
\section{Learning rates, excess risk, and optimal weighting}
\label{sec:rates}

We now make the learning scale explicit and study the effect of the
weight exponent. Recall \(L_n=2\log\log n\) and
\(r_n=B_nL_n^{m/2}/W_n\).

\subsection{Exact learning and excess-risk rates}
\label{subsec:explicit-rate}

\begin{proposition}[Explicit LIL learning scale]
\label{prop:explicit-learning-scale}
Under \textup{(P\ref{P1})--(P\ref{P3})},
\begin{equation}
\label{eq:rn-expanded}
r_n
\sim
(1-\kappa)
\sqrt{m!c^m I_{m,\alpha,\kappa}}\,
n^{-\alpha m/2}
L_0(n)^{m/2}
(2\log\log n)^{m/2}.
\end{equation}
\end{proposition}

\begin{proof}
By \eqref{eq:reference-variance},
\(B_n\sim\sqrt{m!c^mI_{m,\alpha,\kappa}}\,
n^{1-\kappa-\alpha m/2}L_a(n)L_0(n)^{m/2}\), while Karamata's
theorem gives
\(W_n\sim n^{1-\kappa}L_a(n)/(1-\kappa)\).
Their ratio yields \eqref{eq:rn-expanded}.
\end{proof}

Thus \(L_a(n)\) cancels completely, and \(\kappa\) disappears from
the polynomial exponent.

Define
\begin{equation}
\label{eq:weighting-factor}
\mathcal W_{m,\alpha}(\kappa)
:=
(1-\kappa)
\sqrt{m!I_{m,\alpha,\kappa}}\,
\Lambda_{m,\alpha,\kappa},
\qquad
\Lambda_{1,\alpha,\kappa}:=1.
\end{equation}

\begin{theorem}[Exact almost-sure learning rate]
\label{thm:explicit-learning-rate}
Under the assumptions of Theorem~\ref{thm:parameter-LIL},
\begin{equation}
\label{eq:explicit-parameter-limsup}
\limsup_{n\to\infty}
\frac{\|\widehat\theta_n-\theta_0\|}
{n^{-\alpha m/2}L_0(n)^{m/2}
(2\log\log n)^{m/2}}
=
c^{m/2}\|V^{-1}\mathcal J\|
\mathcal W_{m,\alpha}(\kappa)
\qquad\text{a.s.}
\end{equation}
For every \(v\in\mathbb R^p\), the same formula holds with
\(\|V^{-1}\mathcal J\|\) replaced by
\(|v^\top V^{-1}\mathcal J|\) and the numerator by
\(|v^\top(\widehat\theta_n-\theta_0)|\).
\end{theorem}

\begin{proof}
Combine Theorem~\ref{thm:parameter-LIL} with
Proposition~\ref{prop:explicit-learning-scale}.
\end{proof}

Hence the polynomial learning exponent is \(-\alpha m/2\) for every
admissible regularly varying weight: memory and chaos rank determine
the exponent, while the weight affects the sharp constant.

Let
\(\mathscr R_n:=R(\widehat\theta_n)-R(\theta_0)\)
denote the excess population risk.

\begin{theorem}[Exact LIL for excess population risk]
\label{thm:excess-risk-LIL}
Under the assumptions of Theorem~\ref{thm:parameter-LIL},
\begin{equation}
\label{eq:risk-leading-chaos}
\mathscr R_n
=
\frac{\mathcal J^\top V^{-1}\mathcal J}{2W_n^2}
\left(\sum_{t=1}^{n}a_tH_m(X_t)\right)^2
+
o_{\mathrm{a.s.}}(r_n^2),
\end{equation}
and
\begin{equation}
\label{eq:explicit-excess-risk-limsup}
\limsup_{n\to\infty}
\frac{\mathscr R_n}
{n^{-\alpha m}L_0(n)^m(2\log\log n)^m}
=
\frac{c^m}{2}
\mathcal J^\top V^{-1}\mathcal J\,
\mathcal W_{m,\alpha}(\kappa)^2
\qquad\text{a.s.}
\end{equation}
\end{theorem}

\begin{proof}
Taylor expansion at \(\theta_0\) gives
\[
\mathscr R_n
=
\frac12
(\widehat\theta_n-\theta_0)^\top
V(\widehat\theta_n-\theta_0)
+
o_{\mathrm{a.s.}}(r_n^2),
\]
since \(\|\widehat\theta_n-\theta_0\|=O_{\mathrm{a.s.}}(r_n)\).
Theorem~\ref{thm:LIL-Bahadur} then yields
\eqref{eq:risk-leading-chaos}, while the sharp weighted LIL and
Proposition~\ref{prop:explicit-learning-scale} give
\eqref{eq:explicit-excess-risk-limsup}.
\end{proof}

For \(m\ge2\), the excess-risk cluster sets are obtained from the
parameter cluster sets by the continuous quadratic map
\(x\mapsto \frac12x^\top Vx\). In particular, for every fixed \(t>0\),
\begin{equation}
\label{eq:pointwise-excess-risk-envelope}
\limsup_{n\to\infty}
\frac{
R(\widehat\theta_{\lfloor nt\rfloor})-R(\theta_0)
}{r_n^2}
=
\frac12
\mathcal J^\top V^{-1}\mathcal J\,
t^{-\alpha m}
\Lambda_{m,\alpha,\kappa}^2
\qquad\text{a.s.}
\end{equation}
The corresponding functional statement follows from
Theorem~\ref{thm:raw-learning-curve} by the same quadratic map and is
recorded in the Supplement.

\subsection{Dependence of the sharp constant on the weight}
\label{subsec:optimal-weighting-general}

The admissible set is
\(\mathcal D_{m,\alpha}
:=\{\kappa:2\kappa+\alpha m<1\}
=(-\infty,(1-\alpha m)/2)\).
All factors in \eqref{eq:explicit-parameter-limsup} other than
\(\mathcal W_{m,\alpha}(\kappa)\) are independent of \(\kappa\).
Thus optimal power weighting amounts to minimizing
\(\mathcal W_{m,\alpha}\) over \(\mathcal D_{m,\alpha}\).

For \(m\ge2\), \(\mathcal W_{m,\alpha}\) is continuous on
\(\mathcal D_{m,\alpha}\). The two-sided estimate for
\(\Lambda_{m,\alpha,\kappa}\) from the companion paper gives
\begin{equation}
\label{eq:weighting-objective-bounds}
\underline{\mathcal W}_{m,\alpha}(\kappa)
\le
\mathcal W_{m,\alpha}(\kappa)
\le
(1-\kappa)\sqrt{I_{m,\alpha,\kappa}},
\end{equation}
where
\begin{align}
\label{eq:weighting-lower}
\underline{\mathcal W}_{m,\alpha}(\kappa)
&:=
(1-\kappa)
\left(\frac{2-\alpha}{2(1-\alpha)}\right)^{m/2}
\notag\\[-1mm]
&\quad\times
\sum_{j=0}^{m}
\binom{m}{j}
\mathrm B\!\left(
1-\kappa+(1-\alpha)j,\,
1+(1-\alpha)(m-j)
\right).
\end{align}
Continuity and \eqref{eq:weighting-objective-bounds} are proved in the
Supplement.

\subsection{Global optimization for \(m=1\)}
\label{subsec:optimal-weighting-linear}

For \(m=1\), \(\Lambda_{1,\alpha,\kappa}=1\), and minimizing
\(\mathcal W_{1,\alpha}\) is equivalent to minimizing
\begin{equation}
\label{eq:F-alpha-beta}
F_\alpha(\kappa)
:=
\mathcal W_{1,\alpha}(\kappa)^2
=
\frac{
2(1-\kappa)^2
\mathrm B(1-\kappa,1-\alpha)
}{
2-2\kappa-\alpha
},
\qquad
\kappa<\frac{1-\alpha}{2}.
\end{equation}

\begin{theorem}[Global optimal weighting for \(m=1\)]
\label{thm:linear-optimal-weight}
For every \(\alpha\in(0,1)\), the set
\(\mathcal K_\alpha^\star:=
\operatorname*{arg\,min}_{\kappa<(1-\alpha)/2}F_\alpha(\kappa)\)
is nonempty and
\begin{equation}
\label{eq:kappa-star-interior}
\mathcal K_\alpha^\star
\subset
\left(0,\frac{1-\alpha}{2}\right).
\end{equation}
Hence every globally optimal admissible exponent is positive.
Moreover,
\(\mathcal W_{1,\alpha}(\kappa)>
\mathcal W_{1,\alpha}(0)\) for every \(\kappa<0\), while for some
\(\kappa_\alpha^\star\in(0,(1-\alpha)/2)\),
\begin{equation}
\label{eq:uniform-strictly-suboptimal}
\mathcal W_{1,\alpha}(\kappa_\alpha^\star)
<
\mathcal W_{1,\alpha}(0).
\end{equation}
Thus neither negative power exponents nor the unweighted benchmark
\(\kappa=0\) minimize the exact \(m=1\) LIL constant.
\end{theorem}

\begin{proof}
Differentiating \eqref{eq:F-alpha-beta} gives
\begin{equation}
\label{eq:logF-derivative}
D_\alpha(\kappa)
:=
\frac{d}{d\kappa}\log F_\alpha(\kappa)
=
-\frac{2}{1-\kappa}
-\Psi(1-\kappa)
+\Psi(2-\kappa-\alpha)
+\frac{2}{2-2\kappa-\alpha}.
\end{equation}
For \(\kappa\le0\), an integral representation of the digamma
difference gives \(D_\alpha(\kappa)<0\), while
\[
F_\alpha(\kappa)
\sim
\Gamma(1-\alpha)(1-\kappa)^\alpha
\longrightarrow\infty
\qquad(\kappa\to-\infty).
\]
Thus no global minimizer lies in \((-\infty,0]\).

At \(\kappa=0\), the series representation of the digamma function
gives \(D_\alpha(0)<0\), whereas\\
\(\lim_{\kappa\uparrow(1-\alpha)/2}D_\alpha(\kappa)>0\).
Moreover, \(F_\alpha\) extends continuously to
\(\kappa=(1-\alpha)/2\), since there
\(1-\kappa=(1+\alpha)/2>0\) and
\(2-2\kappa-\alpha=1\).
Hence neither \(0\) nor the right endpoint of the extended domain is
a minimizer. Continuity and divergence as \(\kappa\to-\infty\)
therefore imply that every global minimizer lies in
\((0,(1-\alpha)/2)\). The digamma estimates establishing the sign
statements are given in the Supplement.
\end{proof}

The theorem does not assert uniqueness. Numerical minimizers and the
nonlinear weighting problem are studied in
Section~\ref{sec:numerics}. All statements above concern the standing
regime in \textup{(P\ref{P3})}; critical boundaries require different
normalizations.

%\section{Machine-learning examples}
% =====================================================================
% Section 7. Machine-learning examples and verification of assumptions
% =====================================================================

\section{Machine-learning examples and verification of the assumptions}
\label{sec:examples}

We illustrate the general theory on three smooth learning problems.
For each model we identify the population minimizer \(\theta_0\), the
local curvature \(V\), the Hermite rank \(m\), and the low-frequency
coefficient \(\mathcal J\).  Detailed verification of the learning
assumptions and the Gaussian projection calculations are collected in
the Supplement.

\subsection{One-step linear prediction}
\label{subsec:linear-prediction}

Let \(U_t:=X_{t-1}\), \(Y_t:=X_t\), and consider
\(f_\theta(U_t)=\theta U_t\), \(\theta\in\Theta\subset\mathbb R\),
with squared loss
\(\ell(\theta;U_t,Y_t)=\frac12(Y_t-\theta U_t)^2\).
Write \(\rho_1:=\rho(1)\). Then
\[
R(\theta)
=
\frac12\mathbb E(X_t-\theta X_{t-1})^2
=
\frac12(1-\rho_1^2)+\frac12(\theta-\rho_1)^2,
\]
so
\begin{equation}
\label{eq:linear-prediction-target}
\theta_0=\rho_1,
\qquad
V=1.
\end{equation}

\begin{proposition}[One-step prediction]
\label{prop:linear-prediction-rank}
Assume \textup{(P\ref{P1})}, and let \(\Theta\) be a compact interval
containing \(\rho_1\) in its interior. Then the model satisfies
\textup{(L\ref{L1})--(L\ref{L4})}. The score at \(\theta_0\) is
\[
g_t^{\mathrm{pred}}
=
X_{t-1}(\rho_1X_{t-1}-X_t),
\]
belongs entirely to the second Wiener chaos, and satisfies
\begin{equation}
\label{eq:linear-prediction-J}
m=2,
\qquad
\mathcal J_{\mathrm{pred}}=\rho_1-1\ne0.
\end{equation}
Thus \textup{(P\ref{P2})} holds with \(m=2\).
\end{proposition}

\begin{corollary}[Exact learning rate for one-step prediction]
\label{cor:linear-prediction-rate}
Assume \textup{(P\ref{P1})} and
\textup{(P\ref{P3})--(P\ref{P6})}. Then
\begin{equation}
\label{eq:linear-prediction-limsup}
\limsup_{n\to\infty}
\frac{|\widehat\theta_n^{\mathrm{pred}}-\rho_1|}
{n^{-\alpha}L_0(n)(2\log\log n)}
=
c\,|1-\rho_1|\,\mathcal W_{2,\alpha}(\kappa)
\qquad\text{a.s.}
\end{equation}
Moreover,
\begin{equation}
\label{eq:linear-prediction-risk-limsup}
\limsup_{n\to\infty}
\frac{
R(\widehat\theta_n^{\mathrm{pred}})-R(\rho_1)
}{
n^{-2\alpha}L_0(n)^2(2\log\log n)^2
}
=
\frac12c^2(1-\rho_1)^2
\mathcal W_{2,\alpha}(\kappa)^2
\qquad\text{a.s.}
\end{equation}
\end{corollary}

\begin{proof}
Apply Theorems~\ref{thm:explicit-learning-rate} and
\ref{thm:excess-risk-LIL} with \(m=2\), \(V=1\), and
\(\mathcal J=\rho_1-1\).
\end{proof}

Although the predictor is linear in its parameter, its score is
quadratic in the Gaussian data, so the learning behavior is governed
by the \(m=2\) weighted LIL.

\subsection{Least-squares classification of a threshold event}
\label{subsec:threshold-classification}

Fix \(b\in\mathbb R\), let \(U_t:=X_{t-1}\),
\(Y_t:=\mathbf 1_{\{X_t>b\}}\), and fit
\(f_\theta(u)=\theta_1+\theta_2u\),
\(\theta=(\theta_1,\theta_2)^\top\in\Theta\subset\mathbb R^2\),
by squared loss
\(\ell(\theta;U_t,Y_t)
=\frac12(Y_t-\theta_1-\theta_2U_t)^2\).
Let \(\varphi\) be the standard normal density and put
\(p_b:=\mathbb P(X_0>b)=\overline\Phi_{\mathrm N}(b)\).

\begin{proposition}[Threshold classification]
\label{prop:threshold-classification}
Assume \textup{(P\ref{P1})}, and let
\(\Theta\subset\mathbb R^2\) be compact and contain
\[
\theta_0^{\mathrm{cls}}
:=
\begin{pmatrix}
p_b\\
\rho_1\varphi(b)
\end{pmatrix}
\]
in its interior. Then the model satisfies
\textup{(L\ref{L1})--(L\ref{L4})}, with
\begin{equation}
\label{eq:classification-Hessian}
V=I_2.
\end{equation}
Its score has a nonzero first-chaos projection, with
\begin{equation}
\label{eq:classification-J}
m=1,
\qquad
\mathcal J_{\mathrm{cls}}
=
-\varphi(b)
\begin{pmatrix}
1-\rho_1\\
\rho_1b
\end{pmatrix}
\ne0.
\end{equation}
Thus \textup{(P\ref{P2})} holds with \(m=1\).
\end{proposition}

\begin{corollary}[Exact learning rate for threshold classification]
\label{cor:classification-rate}
Assume \textup{(P\ref{P1})} and \textup{(P\ref{P3})}. Then
\begin{equation}
\label{eq:classification-limsup}
\limsup_{n\to\infty}
\frac{
\|\widehat\theta_n^{\mathrm{cls}}-\theta_0^{\mathrm{cls}}\|
}{
n^{-\alpha/2}L_0(n)^{1/2}(2\log\log n)^{1/2}
}
=
c^{1/2}\varphi(b)
\sqrt{(1-\rho_1)^2+\rho_1^2b^2}\,
\mathcal W_{1,\alpha}(\kappa)
\qquad\text{a.s.}
\end{equation}
The excess population risk satisfies
\begin{equation}
\label{eq:classification-risk-limsup}
\limsup_{n\to\infty}
\frac{
R(\widehat\theta_n^{\mathrm{cls}})
-R(\theta_0^{\mathrm{cls}})
}{
n^{-\alpha}L_0(n)(2\log\log n)
}
=
\frac{c}{2}\varphi(b)^2
\bigl((1-\rho_1)^2+\rho_1^2b^2\bigr)
\mathcal W_{1,\alpha}(\kappa)^2
\qquad\text{a.s.}
\end{equation}
\end{corollary}

\begin{proof}
Apply Theorems~\ref{thm:explicit-learning-rate} and
\ref{thm:excess-risk-LIL} with \(m=1\), \(V=I_2\), and
\(\mathcal J=\mathcal J_{\mathrm{cls}}\).
\end{proof}

\begin{corollary}[Weighting effect for the \(m=1\) classifier]
\label{cor:classification-weighting}
For every \(\alpha\in(0,1)\), there exists
\(\kappa_\alpha^\star\in(0,(1-\alpha)/2)\) such that the exact
constants in \eqref{eq:classification-limsup} and
\eqref{eq:classification-risk-limsup} are strictly smaller at
\(\kappa=\kappa_\alpha^\star\) than at \(\kappa=0\).
\end{corollary}

\begin{proof}
The only \(\kappa\)-dependent factor is
\(\mathcal W_{1,\alpha}(\kappa)\); apply
Theorem~\ref{thm:linear-optimal-weight}.
\end{proof}

The same mechanism extends to other smooth classification losses,
such as logistic loss, provided the learning assumptions
\textup{(L\ref{L1})--(L\ref{L4})} hold and the first low-frequency chaos coefficient
is nonzero.
The squared-loss classifier is convenient because both \(V\) and
\(\mathcal J\) are available in closed form.

\subsection{A symmetry-induced change of the learning rate}
\label{subsec:rank-switching}

Fix \(s\in\mathbb R\), define \(Y_t^{(s)}:=(X_t-s)^2\), and estimate
a constant predictor \(f_\theta\equiv\theta\) by squared loss
\(\ell_s(\theta;X_t)
:=\frac12(\theta-(X_t-s)^2)^2\).
The population minimizer is
\begin{equation}
\label{eq:quadratic-target}
\theta_0^{(s)}=1+s^2.
\end{equation}

\begin{proposition}[Rank transition at the symmetric target]
\label{prop:rank-transition}
Let \(\Theta\) be a compact interval containing \(1+s^2\) in its
interior. Then the model satisfies
\textup{(L\ref{L1})--(L\ref{L4})}, with \(V=1\). The score at the
target is
\begin{equation}
\label{eq:quadratic-target-score}
g_s(X_t)
:=
\psi_s(\theta_0^{(s)};X_t)
=
1+s^2-(X_t-s)^2
=
2sH_1(X_t)-H_2(X_t).
\end{equation}
Hence
\begin{equation}
\label{eq:rank-transition-parameters}
s\ne0 \Longrightarrow m=1,\ \mathcal J_s=2s,
\qquad
s=0 \Longrightarrow m=2,\ \mathcal J_0=-1.
\end{equation}
Thus \textup{(P\ref{P2})} holds with the corresponding value of \(m\).
\end{proposition}

\begin{corollary}[Two distinct learning regimes]
\label{cor:rank-transition-rates}
If \(s\ne0\), assume \textup{(P\ref{P1})} and
\textup{(P\ref{P3})}. Then
\begin{equation}
\label{eq:shifted-quadratic-rate}
\limsup_{n\to\infty}
\frac{
|\widehat\theta_n^{(s)}-(1+s^2)|
}{
n^{-\alpha/2}L_0(n)^{1/2}(2\log\log n)^{1/2}
}
=
2|s|\,c^{1/2}\mathcal W_{1,\alpha}(\kappa)
\qquad\text{a.s.}
\end{equation}
If \(s=0\), assume \textup{(P\ref{P1})} and
\textup{(P\ref{P3})--(P\ref{P6})}. Then
\begin{equation}
\label{eq:symmetric-quadratic-rate}
\limsup_{n\to\infty}
\frac{
|\widehat\theta_n^{(0)}-1|
}{
n^{-\alpha}L_0(n)(2\log\log n)
}
=
c\,\mathcal W_{2,\alpha}(\kappa)
\qquad\text{a.s.}
\end{equation}
\end{corollary}

\begin{proof}
Apply Theorem~\ref{thm:explicit-learning-rate} with
\((m,V,\mathcal J)=(1,1,2s)\) for \(s\ne0\), and with
\((m,V,\mathcal J)=(2,1,-1)\) for \(s=0\).
\end{proof}

Thus exact symmetry cancels the first chaos and changes the
almost-sure learning exponent from \(-\alpha/2\) to \(-\alpha\).

By Remark~\ref{rem:parameter-cluster-parity}, the two \(m=2\)
examples also have one-sided endpoint cluster sets. In one-step
prediction,
\[
\operatorname{Clust}_{n\to\infty}
\left\{
\frac{W_n(\widehat\theta_n^{\mathrm{pred}}-\rho_1)}
     {B_nL_n}
\right\}
=
[0,(1-\rho_1)\Lambda_{2,\alpha,\kappa}]
\qquad\text{a.s.},
\]
while for the symmetric quadratic target the corresponding cluster
set is \([0,\Lambda_{2,\alpha,\kappa}]\). By contrast, the \(m=1\)
examples have symmetric endpoint cluster sets.

\begin{table}[htb]
\centering
\small
\renewcommand{\arraystretch}{1.2}
\begin{tabular}{
    @{}
    p{0.25\textwidth}
    c
    c
    p{0.22\textwidth}
    p{0.25\textwidth}
    @{}
}
\toprule
\textbf{Learning problem}
&
\(\boldsymbol{p}\)
&
\(\boldsymbol{m}\)
&
\textbf{Low-frequency coefficient}
&
\textbf{Parameter LIL scale}
\\
\midrule

One-step linear prediction
&
\(1\)
&
\(2\)
&
\(\rho_1-1\)
&
\(n^{-\alpha}L_0(n)(2\log\log n)\)
\\[1mm]

Threshold-event least-squares classification
&
\(2\)
&
\(1\)
&
\(-\varphi(b)\bigl(1-\rho_1,\rho_1b\bigr)^\top\)
&
\(n^{-\alpha/2}L_0(n)^{1/2}\sqrt{2\log\log n}\)
\\[1mm]

Quadratic target, \(s\ne0\)
&
\(1\)
&
\(1\)
&
\(2s\)
&
\(n^{-\alpha/2}L_0(n)^{1/2}\sqrt{2\log\log n}\)
\\[1mm]

Symmetric quadratic target, \(s=0\)
&
\(1\)
&
\(2\)
&
\(-1\)
&
\(n^{-\alpha}L_0(n)(2\log\log n)\)
\\

\bottomrule
\end{tabular}
\caption{
Machine-learning examples: parameter dimension \(p\), chaos rank \(m\),
effective low-frequency coefficient, and almost-sure learning scale.
The corresponding sharp constants are given in
Corollaries~\ref{cor:linear-prediction-rate},
\ref{cor:classification-rate}, and
\ref{cor:rank-transition-rates}.
}
\label{tab:ML-examples}
\end{table}
%\section{Numerical experiments}
% =====================================================================
% % =====================================================================
% Section 8. Numerical experiments
% =====================================================================
\section{Numerical experiments}
\label{sec:numerics}

We illustrate two aspects of the theory: Section~8.1 reports deterministic numerical
evaluation of the sharp weighting objective, and Section~8.2 presents Monte Carlo behavior
of the learning error, including its rank-one geometry and the
symmetry-induced chaos-rank transition.

The Gaussian input is fractional Gaussian noise,
\(X_t=B_H(t+1)-B_H(t)\), with \(H=1-\alpha/2\), generated by the
Davies--Harte circulant-embedding method~\cite{DaviesHarte1987}.  The
Monte Carlo experiments use pure power weights \(a_t=t^{-\kappa}\),
\(N=2^{16}\), \(M=120\) independent trajectories, and the dyadic grid
\(n=2^8,\ldots,2^{16}\).  The classification experiment uses seed
\(\mathtt{20260830}\), while the rank-transition experiment uses seed
\(\mathtt{20260831}\).

\subsection{Weighting the sharp LIL constant}
\label{subsec:numerical-weighting}

For \(m\ge2\), the weighting objective
\(\mathcal W_{m,\alpha}(\kappa)\) also contains the variational constant
\(\Lambda_{m,\alpha,\kappa}\).  We compute it from the reduced
\([0,1]\) variational problem used in the companion paper, with exact
cell integration and shifted symmetric power iteration.  The main
computations use 1500 cells; selected values were checked on grids up
to 6000 cells.

At \(\alpha=0.2\), the computed minimizers under this numerical scheme are
\(\kappa^\star=0.1361,0.1310,0.1250\) for \(m=1,2,3\), respectively.
The corresponding reductions relative to uniform weighting are
\(0.100\%\), \(0.205\%\), and \(0.315\%\).  Thus the best positive
weighting gives only a modest improvement.  The opposite direction is
substantially more costly: at \(\kappa=-0.6\), the sharp constant is
larger than at \(\kappa=0\) by \(1.44\%\), \(3.10\%\), and \(5.04\%\)
for \(m=1,2,3\).

\begin{figure}[t]
\centering
\includegraphics[width=0.49\textwidth]{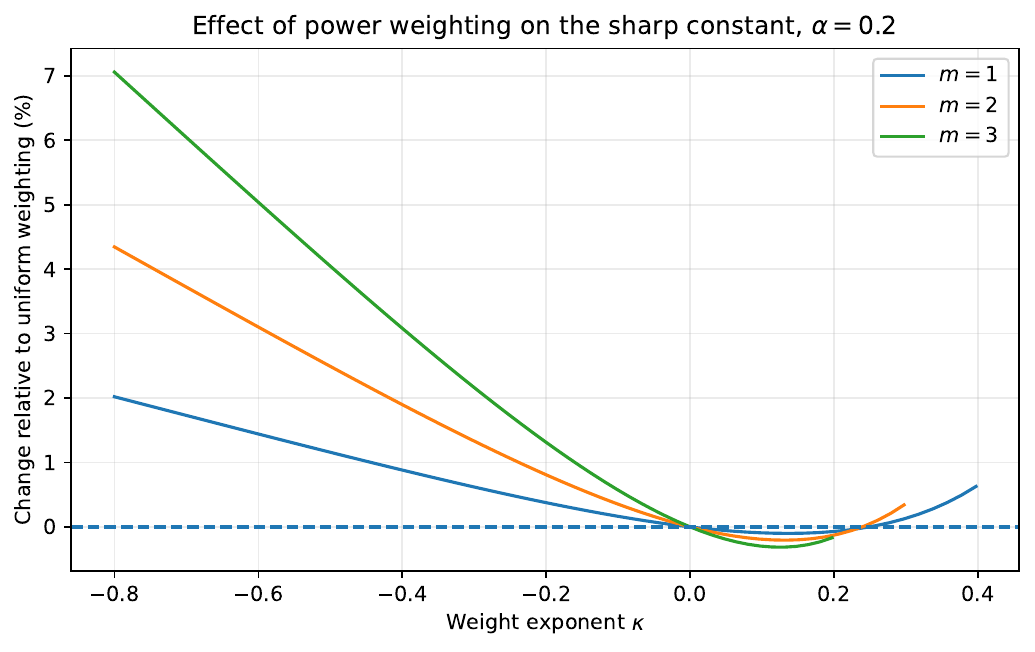}
\hfill
\includegraphics[width=0.49\textwidth]{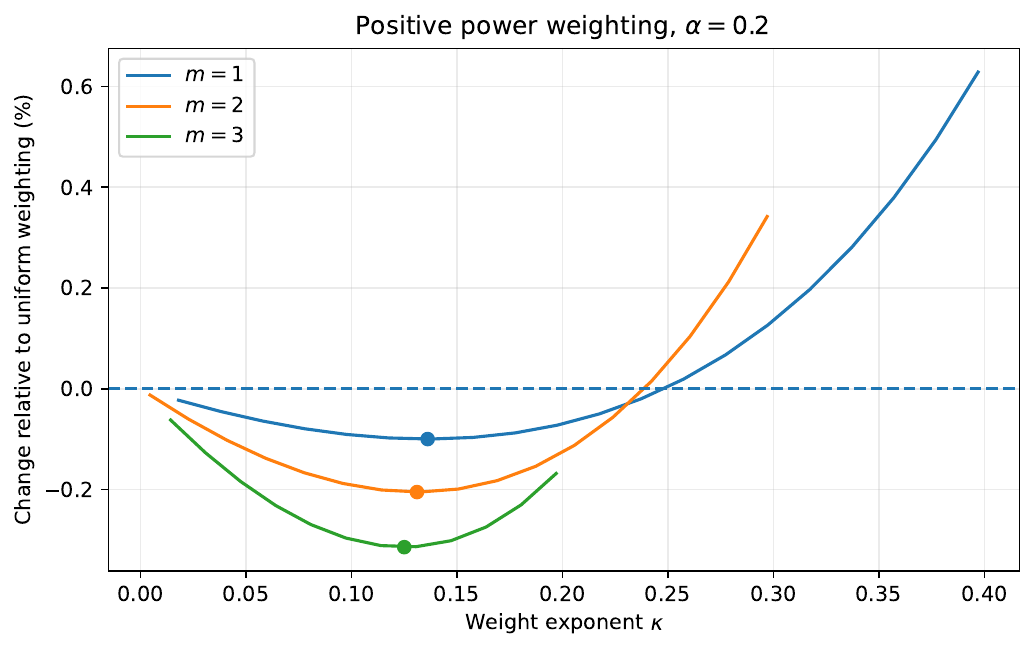}
\caption{
Effect of power weighting on the sharp LIL constant at \(\alpha=0.2\).
Left: percentage change relative to uniform weighting over a broad
range of \(\kappa\).  Negative exponents, which favor later
observations in the fixed-origin scheme, can produce a substantially
larger deterioration than the gain obtained from positive weighting.
Right: enlargement of the positive range; the points mark the
numerical minima for \(m=1,2,3\).
}
\label{fig:weighting-profiles}
\end{figure}

For stronger effective memory, the nonlinear objective can behave
qualitatively differently from the rank-one case: for example, at
\((m,\alpha)=(2,0.45)\) and \((3,0.30)\) it decreases numerically up to
the admissible boundary
\(\kappa=(1-\alpha m)/2\).  This numerical behavior suggests that the nonlinear problem may have
a boundary infimum rather than an interior minimizer. Grid-refinement
checks and the numerical localization of this transition are given in
the supplementary material.

\subsection{Learning-error geometry and chaos-rank transition}
\label{subsec:numerical-learning-geometry}

We first use the threshold-classification example with
\(\alpha=0.2\), \(b=0.5\), and
\(\kappa\in\{-0.15,0,0.10,0.20\}\).  Let
\(e_{\parallel}\) be the unit vector in the leading direction
\(V^{-1}\mathcal J\), and let \(e_{\perp}\) be an orthogonal unit
vector.  For all four weights,
\(\operatorname{RMSE}_{\perp}/\operatorname{RMSE}_{\parallel}\)
decreases from about \(0.46\) at \(n=2^8\) to about \(0.17\) at
\(n=2^{16}\).  The fitted log--log slopes of this ratio over the final
five dyadic points lie between \(-0.163\) and \(-0.175\).  This
provides finite-sample evidence for the rank-one asymptotic geometry of
the parameter error.

We next consider the quadratic-target estimator of
Section~\ref{sec:examples} with \(\alpha=0.4\), \(\kappa=0\), and
\(s\in\{0,0.1,0.2,0.4,0.6\}\).  Since
\(
    (X_t-s)^2-(1+s^2)=H_2(X_t)-2sH_1(X_t),
\)
the symmetry point \(s=0\) has rank two and asymptotic exponent
\(-\alpha=-0.4\), whereas every fixed \(s\ne0\) has rank one and
eventual exponent \(-\alpha/2=-0.2\).  The fitted slopes are
\(-0.370,-0.328,-0.282,-0.245,-0.231\) as \(s\) increases from \(0\)
to \(0.6\), showing the predicted finite-sample crossover toward the
rank-one regime.

\begin{figure}[t]
\centering
\includegraphics[width=0.49\textwidth]{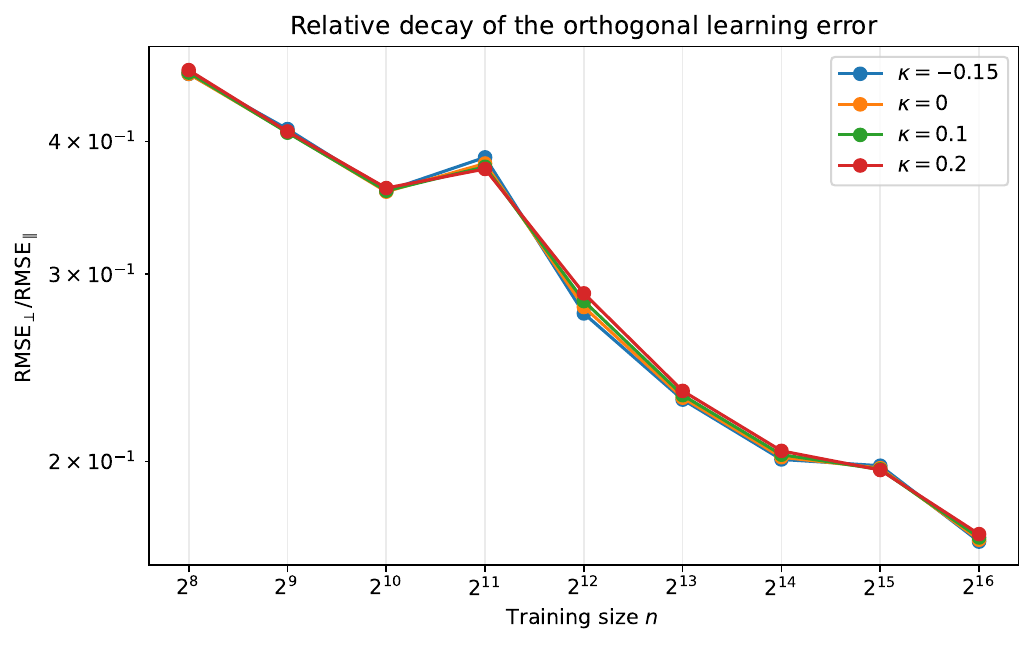}
\hfill
\includegraphics[width=0.49\textwidth]{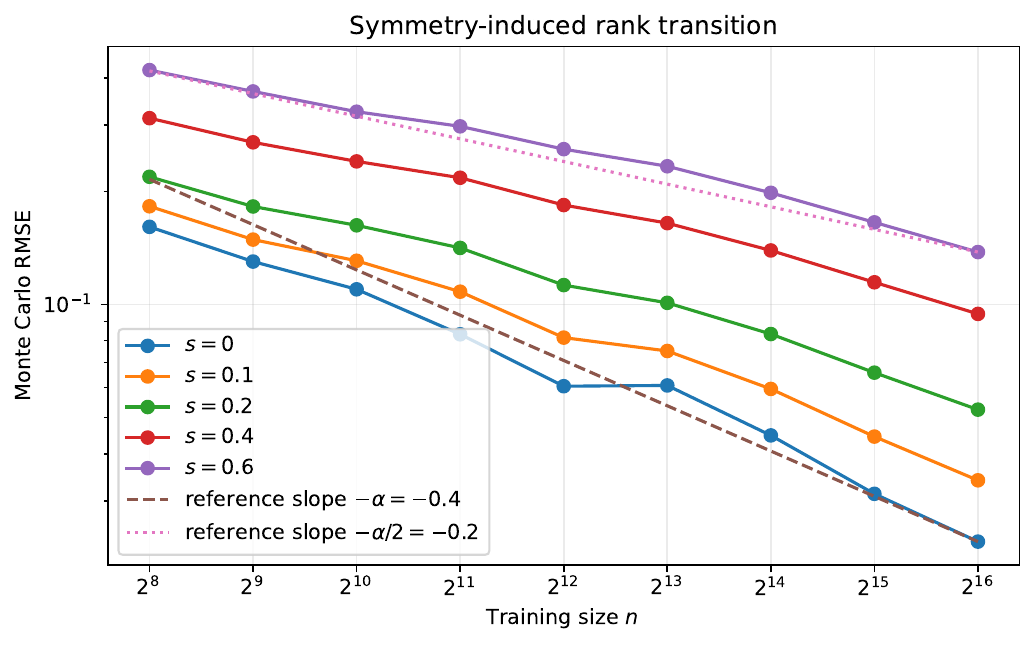}
\caption{
Finite-sample manifestations of the chaos geometry.
Left: ratio of the orthogonal and parallel RMSEs in threshold
classification; its decay illustrates concentration of the learning
error along \(V^{-1}\mathcal J\).
Right: symmetry-induced rank transition in the quadratic-target model.
At \(s=0\) the rank-two slope is already visible, while increasing
\(s\) makes the rank-one component dominate earlier.
}
\label{fig:learning-geometry}
\end{figure}

The total classification RMSE has nearly the same fitted slope for all
four values of \(\kappa\), although the common finite-sample slope is
still pre-asymptotic.  The corresponding RMSE curves, fitted slopes,
and all grid-refinement checks are reported in the supplementary
material.
%\section{Discussion and extensions}
% =====================================================================
% Section 9. Discussion
% =====================================================================

\section{Discussion}
\label{sec:discussion}

The results give an exact pathwise learning theory for regularly
weighted empirical-risk minimization under long-range dependence.
The companion weighted LIL~\cite{Moldavskaya2026WeightedLIL} provides
the probabilistic input; the present paper identifies how its
fluctuations propagate through the loss gradient and the population
curvature to the learned parameter, the learning trajectory, and the
excess expected loss.

The leading Wiener chaos of the loss gradient determines the
polynomial learning exponent.  If its rank is \(m\), then the parameter
error has exponent \(-\alpha m/2\), while the excess expected loss has
exponent \(-\alpha m\).  The local curvature transforms the leading
score direction \(\mathcal J\) into \(V^{-1}\mathcal J\); consequently,
for scalar Gaussian input the leading multidimensional learning
trajectory has rank-one pathwise geometry.

A second conclusion concerns sample weighting.  For
\(a_t=t^{-\kappa}L_a(t)\), the slowly varying factor \(L_a\) cancels
from the leading learning scale, and \(\kappa\) does not change its
polynomial exponent.  It does, however, change the exact LIL constant
and, in the nonlinear regime, the cluster geometry.  Thus two
weighting rules can have the same conventional learning rate while
producing different magnitudes of the largest persistent pathwise
fluctuations.

For \(m=1\), this comparison can be made globally over the admissible
range \(\kappa<(1-\alpha)/2\).
Theorem~\ref{thm:linear-optimal-weight} shows that every global
minimizer is strictly positive.  Hence both the unweighted choice
\(\kappa=0\) and negative exponents are suboptimal for the exact
rank-one LIL constant.  Negative \(\kappa\) favor later observations
within the fixed-origin weighting scheme considered here.  They should
not be confused with age-dependent forgetting weights
\(a_{n,t}=w(n-t)\) or \(w(t/n)\), which form triangular arrays and
require a different theory.  For \(m\ge2\), optimization is potentially
richer because \(\Lambda_{m,\alpha,\kappa}\) also depends on the
weight exponent.

The LIL is an almost-sure asymptotic statement and does not by itself
provide finite-sample confidence sequences or other anytime-valid
guarantees.  Such results would require additional time-uniform
probability control.  Likewise, the present results concern empirical
risk minimization and do not automatically extend to stochastic
gradient descent: an SGD error recursion contains products of
linearized update operators and is not a deterministic weighted sum
of gradients evaluated at \(\theta_0\).

The theory is proved in the noncentral interior regime
\(
    \alpha m<1,
    \quad
    2\kappa+\alpha m<1.
\)
The critical cases \(\alpha m=1\) and
\(2\kappa+\alpha m=1\) require different logarithmic normalizations.
When \(\alpha m>1\), the leading covariance becomes summable under
appropriate additional conditions and a short-memory LIL would require
a separate strong-approximation or invariance-principle argument.

Several extensions are natural.  A vector-valued Gaussian driver may
produce genuinely multidimensional cluster geometry; growing lag
windows would require quantitative control of the low-frequency lag
reduction; and nonsmooth losses or nonidentifiable, overparameterized
models require different local expansions.  Data-dependent sample
weights are another distinct direction because they destroy the
deterministic weighted-sum structure used here.

In summary, long memory and chaos rank determine the polynomial
learning exponent, whereas regularly varying sample weights determine
the sharp pathwise constant and the detailed cluster geometry.  This
separation provides a statistical-learning interpretation of the
weighted LIL and a principled way to compare power-law weighting
schemes for long-range dependent data.
%% ================================================================
%% Acknowledgments
%% ================================================================

%% ================================================================
%% Supplementary material
%% ================================================================

\begin{supplement}
\stitle{Supplementary proofs and numerical details}
\sdescription{The supplement contains proofs for weighted empirical
risk minimization, functional delta-method and excess-risk results,
weighting estimates, verification of the learning examples, and
additional numerical methods and results.}
\end{supplement}

%% ================================================================
%% Bibliography
%% ================================================================


\begin{thebibliography}{99}

\bibitem[Arcones(1994)]{Arcones1994}
Arcones, M. A. (1994).
Limit theorems for nonlinear functionals of a stationary Gaussian
sequence of vectors.
\emph{The Annals of Probability} \textbf{22}, 2242--2274.

\bibitem[Benning, Nourdin and Peccati(2026)]{BenningNourdinPeccati2026}
Benning, F., Nourdin, I. and Peccati, G. (2026).
Correlated initialization of deep residual networks.
Preprint, arXiv:2609.03589.

\bibitem[Bingham, Goldie and Teugels(1987)]{BGT1987}
Bingham, N. H., Goldie, C. M. and Teugels, J. L. (1987).
\emph{Regular Variation}.
Encyclopedia of Mathematics and its Applications 27.
Cambridge University Press, Cambridge.

\bibitem[Cortes, Mansour and Mohri(2010)]{CortesMansourMohri2010}
Cortes, C., Mansour, Y. and Mohri, M. (2010).
Learning bounds for importance weighting.
In \emph{Advances in Neural Information Processing Systems 23},
442--450.

\bibitem[Davies and Harte(1987)]{DaviesHarte1987}
Davies, R. B. and Harte, D. S. (1987).
Tests for Hurst effect.
\emph{Biometrika} \textbf{74}, 95--101.

\bibitem[Dehling and Taqqu(1989)]{DehlingTaqqu1989}
Dehling, H. and Taqqu, M. S. (1989).
The empirical process of some long-range dependent sequences with an
application to U-statistics.
\emph{The Annals of Statistics} \textbf{17}, 1767--1783.

\bibitem[Fan, Yan and Xiu(2014)]{FanYanXiu2014}
Fan, J., Yan, A. and Xiu, N. (2014).
Asymptotic properties for \(M\)-estimators in linear models with
dependent random errors.
\emph{Journal of Statistical Planning and Inference}
\textbf{148}, 49--66.

\bibitem[He and Wang(1995)]{HeWang1995}
He, X. and Wang, G. (1995).
Law of the iterated logarithm and invariance principle for
\(M\)-estimators.
\emph{Proceedings of the American Mathematical Society}
\textbf{123}, 563--573.

\bibitem[Ivanov et al.(2013)]{IvanovEtAl2013}
Ivanov, A. V., Leonenko, N. N., Ruiz-Medina, M. D. and Savich, I. N.
(2013).
Limit theorems for weighted nonlinear transformations of Gaussian
stationary processes with singular spectra.
\emph{The Annals of Probability} \textbf{41}, 1088--1114.

\bibitem[Koul(1992)]{Koul1992}
Koul, H. L. (1992).
M-estimators in linear models with long range dependent errors.
\emph{Statistics \& Probability Letters} \textbf{14}, 153--164.

\bibitem[Kuznetsov and Mohri(2017)]{KuznetsovMohri2017}
Kuznetsov, V. and Mohri, M. (2017).
Generalization bounds for non-stationary mixing processes.
\emph{Machine Learning} \textbf{106}, 93--117.

\bibitem[Mohri and Rostamizadeh(2010)]{MohriRostamizadeh2010}
Mohri, M. and Rostamizadeh, A. (2010).
Stability bounds for stationary \(\phi\)-mixing and \(\beta\)-mixing
processes.
\emph{Journal of Machine Learning Research} \textbf{11}, 789--814.

\bibitem[Moldavskaya(2007)]{Moldavskaya2007LSE}
Moldavskaya, E. (2007).
Asymptotic distributions of non-linear inequality constrained least
squares estimation of linear regression coefficients with strong
dependence.
\emph{Theory of Probability and Mathematical Statistics} \textbf{75},
121--137.

\bibitem[Moldavskaya(2026a)]{Moldavskaya2026WeightedLIL}
Moldavskaya, E. (2026a).
Law of Iterated Logarithm for Weighted Sums of Functionals of
Long-Memory Gaussian Sequences.
Preprint, arXiv:2606.21006.

\bibitem[Moldavskaya(2026b)]{Moldavskaya2026MLSupplement}
Moldavskaya, E. (2026b).
Supplement to ``Weighted Iterated-Logarithm Laws for Machine
Learning under Long-Range Dependence.''
\bibitem[Schreuder, Brunel and Dalalyan(2020)]{SchreuderBrunelDalalyan2020}
Schreuder, N., Brunel, V.-E. and Dalalyan, A. (2020).
A nonasymptotic law of iterated logarithm for general M-estimators.
In \emph{Proceedings of the Twenty Third International Conference on
Artificial Intelligence and Statistics}, Proceedings of Machine
Learning Research \textbf{108}, 1331--1341.

\bibitem[Shimodaira(2000)]{Shimodaira2000}
Shimodaira, H. (2000).
Improving predictive inference under covariate shift by weighting the
log-likelihood function.
\emph{Journal of Statistical Planning and Inference} \textbf{90},
227--244.

\bibitem[Sugiyama, Krauledat and M\"uller(2007)]{SugiyamaKrauledatMueller2007}
Sugiyama, M., Krauledat, M. and M\"uller, K.-R. (2007).
Covariate shift adaptation by importance weighted cross validation.
\emph{Journal of Machine Learning Research} \textbf{8}, 985--1005.

\bibitem[Taqqu(1977)]{Taqqu1977}
Taqqu, M. S. (1977).
Law of the iterated logarithm for sums of non-linear functions of
Gaussian variables that exhibit a long range dependence.
\emph{Zeitschrift f\"ur Wahrscheinlichkeitstheorie und Verwandte
Gebiete} \textbf{40}, 203--238.

\end{thebibliography}
\end{document}

\begin{frontmatter}

\title{Supplement to ``Weighted Empirical Risk Minimization for Machine Learning under Long-Range Dependence: Exact Pathwise Rates and Learning-Error Geometry''}

\runtitle{Supplement: Weighted LIL for Long-Memory Learning}

\begin{aug}
\author{\inits{E.}\fnms{Elina}~\snm{Moldavskaya}}
\end{aug}

\end{frontmatter}

\section{Weighted empirical-risk minimization}
\label{supsec:weighted-ERM}

Throughout the Supplement, we use the notation and assumptions of the
main paper.  This section provides the details for Proposition~4.1 of
the main paper.

\subsection{Regularly weighted ergodic averages}
\label{supsubsec:weighted-ergodic}

We first record the weighted ergodic result used for the empirical
risk and the empirical Hessian.

\begin{lemma}[Regularly weighted ergodic lemma]
\label{supp:lem-weighted-ergodic}
Let \(\{Y_t\}_{t\in\mathbb Z}\) be stationary and ergodic with
\(\mathbb E|Y_0|<\infty\). Let
\(a_t=t^{-\kappa}L_a(t)\) be eventually positive and regularly varying,
with \(\kappa<1\). For any fixed positive integer \(t_0\), put
\[
A_n^{(t_0)}:=\sum_{t=t_0}^{n}a_t.
\]
Then
\[
\frac{1}{A_n^{(t_0)}}
\sum_{t=t_0}^{n}a_tY_t
\longrightarrow
\mathbb EY_0
\qquad\text{almost surely}.
\]
\end{lemma}

\begin{proof}
Put
\[
\mu:=\mathbb EY_0,
\qquad
S_n^Y:=\sum_{t=t_0}^{n}(Y_t-\mu).
\]
By the ergodic theorem,
\[
S_n^Y=o(n)
\qquad\text{almost surely}.
\]

By the smooth-variation theorem for regularly varying functions
\citep{BGT1987}, there exists an eventually positive continuously
differentiable function \(\widetilde a\) such that
\[
\widetilde a(t)\sim a_t,
\qquad
\frac{x\widetilde a'(x)}{\widetilde a(x)}
\longrightarrow-\kappa .
\]
Write
\[
\widetilde a_t:=\widetilde a(t),
\qquad
\widetilde A_n^{(t_0)}
:=
\sum_{t=t_0}^{n}\widetilde a_t.
\]
For all sufficiently large \(t\),
\[
|\widetilde a_{t+1}-\widetilde a_t|
\le
C\frac{\widetilde a_t}{t},
\]
while Karamata's theorem gives
\[
\widetilde A_n^{(t_0)}
\sim
\frac{n\widetilde a_n}{1-\kappa}.
\]

Summation by parts yields
\[
\sum_{t=t_0}^{n}
\widetilde a_t(Y_t-\mu)
=
\widetilde a_nS_n^Y
+
\sum_{t=t_0}^{n-1}
(\widetilde a_t-\widetilde a_{t+1})S_t^Y .
\]
Fix \(\varepsilon>0\). Almost surely,
\(|S_t^Y|\le\varepsilon t\) for all sufficiently large \(t\).
Consequently,
\[
\left|
\sum_{t=t_0}^{n}
\widetilde a_t(Y_t-\mu)
\right|
\le
O(1)
+
\varepsilon n\widetilde a_n
+
C\varepsilon
\sum_{t=t_0}^{n-1}\widetilde a_t .
\]
After division by \(\widetilde A_n^{(t_0)}\), Karamata's theorem gives
\[
\limsup_{n\to\infty}
\frac{1}{\widetilde A_n^{(t_0)}}
\left|
\sum_{t=t_0}^{n}
\widetilde a_t(Y_t-\mu)
\right|
\le C\varepsilon
\qquad\text{almost surely}.
\]
Since \(\varepsilon\) is arbitrary,
\[
\frac{1}{\widetilde A_n^{(t_0)}}
\sum_{t=t_0}^{n}
\widetilde a_t(Y_t-\mu)
\longrightarrow0
\qquad\text{almost surely}.
\]

It remains to return from \(\widetilde a_t\) to \(a_t\).
The preceding smooth-weight argument, applied to the stationary
ergodic integrable sequence \(\{|Y_t-\mu|\}\), gives
\[
\frac{1}{\widetilde A_n^{(t_0)}}
\sum_{t=t_0}^{n}
\widetilde a_t|Y_t-\mu|
\longrightarrow
\mathbb E|Y_0-\mu|
\qquad\text{almost surely}.
\]
For fixed \(N\), put
\[
\delta_N
:=
\sup_{t\ge N}
\left|
\frac{a_t}{\widetilde a_t}-1
\right|,
\qquad
\delta_N\longrightarrow0.
\]
Then
$$
\frac{1}{\widetilde A_n^{(t_0)}}
\left|
\sum_{t=t_0}^{n}
(a_t-\widetilde a_t)(Y_t-\mu)
\right|
\le
o(1)
+
\delta_N
\frac{1}{\widetilde A_n^{(t_0)}}
\sum_{t=N}^{n}
\widetilde a_t|Y_t-\mu|.
$$
Taking first \(n\to\infty\) and then \(N\to\infty\) shows that this
difference converges to zero almost surely. Finally,
\(A_n^{(t_0)}/\widetilde A_n^{(t_0)}\to1\), again by Karamata's
theorem and \(a_n/\widetilde a_n\to1\). Hence
\[
\frac{1}{A_n^{(t_0)}}
\sum_{t=t_0}^{n}a_t(Y_t-\mu)
\longrightarrow0
\qquad\text{almost surely},
\]
which proves the result.
\end{proof}

We also need a uniform version over compact parameter sets.

\begin{lemma}[Weighted uniform ergodic lemma]
\label{supp:lem-weighted-uniform-ergodic}
Let the weights be as in Lemma~\ref{supp:lem-weighted-ergodic}.
Let \(K\subset\mathbb R^p\) be compact, and let
\(\{F_\theta(Z_t):t\in\mathbb Z\}\) be stationary and ergodic for
every \(\theta\in K\). Suppose that there exist stationary integrable
sequences \(\{M_t\}\) and \(\{L_t\}\), with \(\{L_t\}\) ergodic,
such that
\[
\sup_{\theta\in K}|F_\theta(Z_t)|\le M_t
\]
and
\[
|F_\theta(Z_t)-F_{\theta'}(Z_t)|
\le
L_t\|\theta-\theta'\|,
\qquad
\theta,\theta'\in K.
\]
Then, for every fixed positive integer \(t_0\),
\[
\sup_{\theta\in K}
\left|
\frac{1}{A_n^{(t_0)}}
\sum_{t=t_0}^{n}a_tF_\theta(Z_t)
-
\mathbb EF_\theta(Z_0)
\right|
\longrightarrow0
\qquad\text{almost surely}.
\]
\end{lemma}

\begin{proof}
By the envelope bound and \(A_n^{(t_0)}\to\infty\), changing finitely
many weights does not affect the conclusion. We may therefore assume
that all weights are positive.
Fix \(\delta>0\), and let
\(\theta_1,\ldots,\theta_{N_\delta}\) be a finite
\(\delta\)-net of \(K\). By
Lemma~\ref{supp:lem-weighted-ergodic}, almost surely,
\[
\frac{1}{A_n^{(t_0)}}
\sum_{t=t_0}^{n}a_tF_{\theta_j}(Z_t)
\longrightarrow
\mathbb EF_{\theta_j}(Z_0)
\]
for every \(j\), and
\[
\frac{1}{A_n^{(t_0)}}
\sum_{t=t_0}^{n}a_tL_t
\longrightarrow
\mathbb EL_0.
\]
For every \(\theta\in K\), choose \(j\) with
\(\|\theta-\theta_j\|\le\delta\). Then
$$
\left|
\frac{1}{A_n^{(t_0)}}
\sum_{t=t_0}^{n}a_tF_\theta(Z_t)
-
\mathbb EF_\theta(Z_0)
\right|
\le
\left|
\frac{1}{A_n^{(t_0)}}
\sum_{t=t_0}^{n}a_tF_{\theta_j}(Z_t)
-
\mathbb EF_{\theta_j}(Z_0)
\right|
+
\delta
\left\{
\frac{1}{A_n^{(t_0)}}
\sum_{t=t_0}^{n}a_tL_t
+
\mathbb EL_0
\right\}.
$$
Taking the supremum over \(\theta\in K\), then the limit superior in
\(n\), and finally letting \(\delta\downarrow0\), proves the claim.
\end{proof}

\subsection{Proof of Proposition 4.1}
\label{supsubsec:proof-prop41}

We now prove the uniform laws, consistency, and local-curvature
statements collected in Proposition~4.1 of the main paper.

\begin{proof}[Detailed proof of Proposition~4.1 of the main paper]
By \textup{(P1)}, \(\rho(h)\to0\). Hence the stationary Gaussian
sequence \(\{X_t\}\) is ergodic, and so are the finite-lag sequence
\(\{\mathbf X_t^{(r)}\}\) and every measurable factor \(\{Z_t\}\).

Under \textup{(P3)}, \(\kappa<1/2\), hence in particular
\(\kappa<1\), so
Lemma~\ref{supp:lem-weighted-ergodic} applies with
\(t_0=r+1\). In this case
\[
A_n^{(r+1)}
=
\sum_{t=r+1}^{n}a_t
=
W_n.
\]

We first prove the uniform law for the empirical risk. Assumption
\textup{(L3)} gives
\[
\Xi_0,\Xi_\ell\in L^2(\mu_r)\subset L^1(\mu_r),
\]
and therefore the required envelopes are integrable. Moreover,
\[
\sup_{\theta\in\Theta}
|\ell(\theta;Z_t)|
\le
\Xi_0(\mathbf X_t^{(r)})
\]
and
\[
|\ell(\theta;Z_t)-\ell(\theta';Z_t)|
\le
\Xi_\ell(\mathbf X_t^{(r)})
\|\theta-\theta'\|.
\]
Since \(\Theta\) is compact,
Lemma~\ref{supp:lem-weighted-uniform-ergodic} yields
\[
\sup_{\theta\in\Theta}
|\widehat R_n(\theta)-R(\theta)|
\longrightarrow0
\qquad\text{almost surely}.
\]

We next prove consistency. Fix \(\varepsilon>0\). By
\textup{(L1)},
\[
\eta_\varepsilon
:=
\inf_{\substack{\theta\in\Theta\\
                 \|\theta-\theta_0\|\ge\varepsilon}}
\{R(\theta)-R(\theta_0)\}
>0.
\]
Almost surely, for all sufficiently large \(n\),
\[
\sup_{\theta\in\Theta}
|\widehat R_n(\theta)-R(\theta)|
<
\frac{\eta_\varepsilon}{3}.
\]
For every \(\theta\in\Theta\) with
\(\|\theta-\theta_0\|\ge\varepsilon\),
$$
\widehat R_n(\theta)
>
R(\theta)-\frac{\eta_\varepsilon}{3}
\ge
R(\theta_0)+\frac{2\eta_\varepsilon}{3}
>
\widehat R_n(\theta_0).
$$
Hence every minimizer \(\widehat\theta_n\) lies in the
\(\varepsilon\)-neighborhood of \(\theta_0\) for all sufficiently
large \(n\). Since \(\varepsilon\) is arbitrary,
\[
\widehat\theta_n\longrightarrow\theta_0
\qquad\text{almost surely}.
\]

We now establish the uniform Hessian law. For
\(\theta\in\Theta_0\), put
\[
H(\theta,\mathbf x)
:=
\nabla_\theta\psi(\theta;\Phi(\mathbf x)).
\]
Let \(K\subset\Theta_0\) be compact. Assumption \textup{(L3)} gives
\[
\sup_{\theta\in K}
\|H(\theta,\mathbf x)\|
\le
\Xi_2(\mathbf x)
\]
and
\[
\|H(\theta,\mathbf x)-H(\theta',\mathbf x)\|
\le
\Xi_H(\mathbf x)\|\theta-\theta'\|,
\qquad
\theta,\theta'\in K,
\]
where
\(\Xi_2,\Xi_H\in L^2(\mu_r)\subset L^1(\mu_r)\).
Applying
Lemma~\ref{supp:lem-weighted-uniform-ergodic}
coordinatewise to the finitely many matrix entries gives
\[
\sup_{\theta\in K}
\|\widehat V_n(\theta)-V(\theta)\|
\longrightarrow0
\qquad\text{almost surely},
\]
where
\[
V(\theta)
=
\mathbb E\nabla_\theta\psi(\theta;Z_0).
\]
Moreover,
\[
\|V(\theta)-V(\theta')\|
\le
\mathbb E\Xi_H(\mathbf X_0^{(r)})
\|\theta-\theta'\|,
\]
so \(V(\theta)\) is continuous on \(\Theta_0\).

Since \(\Theta_0\) is open and contains \(\theta_0\), choose
\(\delta>0\) such that
\[
K_0
:=
\overline B(\theta_0,\delta)
\subset
\Theta_0.
\]
By consistency, almost surely
\(\widehat\theta_n\in K_0\) for all sufficiently large \(n\).
Since \(K_0\) is convex, the entire segment
\[
\{\theta_0+u(\widehat\theta_n-\theta_0):0\le u\le1\}
\]
is then contained in \(K_0\).

Recall
\[
\Delta_n:=\widehat\theta_n-\theta_0
\]
and
\[
\overline V_n
:=
\int_0^1
\widehat V_n(\theta_0+u\Delta_n)\,du.
\]
We have
\begin{align*}
\|\overline V_n-V\|
&\le
\sup_{\theta\in K_0}
\|\widehat V_n(\theta)-V(\theta)\|
\\
&\quad+
\sup_{0\le u\le1}
\|V(\theta_0+u\Delta_n)-V(\theta_0)\|,
\end{align*}
where \(V:=V(\theta_0)\). The first term tends to zero by the uniform
Hessian law, and the second tends to zero by consistency and
continuity of \(V(\theta)\). Hence
\[
\overline V_n\longrightarrow V
\qquad\text{almost surely}.
\]
Since \(V\) is positive definite by \textup{(L4)},
\(\overline V_n\) is invertible for all sufficiently large \(n\), and
continuity of matrix inversion gives
\[
\overline V_n^{-1}
\longrightarrow
V^{-1}
\qquad\text{almost surely}.
\]

Finally, since
\(\theta_0\in\operatorname{int}\Theta\) and
\(\widehat\theta_n\to\theta_0\), almost surely
\(\widehat\theta_n\in\operatorname{int}\Theta\) for all sufficiently
large \(n\). Therefore
\[
\nabla\widehat R_n(\widehat\theta_n)=0.
\]
The integral Taylor formula gives
$$
0
=
\nabla\widehat R_n(\theta_0)
+
\left[
\int_0^1
\widehat V_n(\theta_0+u\Delta_n)\,du
\right]\Delta_n
=
\frac{1}{W_n}
\sum_{t=r+1}^{n}
a_tg(\mathbf X_t^{(r)})
+
\overline V_n\Delta_n
=
\frac{S_n}{W_n}
+
\overline V_n\Delta_n.
$$
Solving for \(\Delta_n\) yields the exact identity
\[
\Delta_n
=
-\,W_n^{-1}\overline V_n^{-1}S_n.
\]
This proves all assertions of Proposition~4.1.
\end{proof}

\section{Pathwise delta method for smooth functionals}
\label{supsec:delta-method}

We now give the detailed pathwise delta-method statement promised at
the end of Section~5 of the main paper. Recall that throughout that
section \(m\ge2\), and let \(\mathcal D_n\) denote the normalized
learning curve of Theorem~5.2 of the main paper. Thus, for fixed
\(0<\varepsilon<T\), \(\mathcal D_n\) is the linear interpolation on
\([\varepsilon,T]\) of
\[
\mathcal D_n(k/n)
=
\frac{\widehat\theta_k-\theta_0}{r_n},
\qquad
r_n=\frac{B_nL_n^{m/2}}{W_n}.
\]
Also recall
\[
q_0=V^{-1}\mathcal J.
\]

\begin{proposition}[Pathwise delta method]
\label{supp:prop-pathwise-delta}
Assume the assumptions of Theorem~5.2 of the main paper. Let
\(F:\Theta_0\to\mathbb R^s\) be continuously differentiable, and fix
\(0<\varepsilon<T\). For all sufficiently large \(n\), let
\(\mathcal D_n^F\) be the linear interpolation on
\([\varepsilon,T]\) of the mesh values
\[
\mathcal D_n^F(k/n)
:=
\frac{F(\widehat\theta_k)-F(\theta_0)}{r_n}.
\]
Then \(\{\mathcal D_n^F\}\) is almost surely relatively compact in
\(C([\varepsilon,T];\mathbb R^s)\), and
\begin{equation}
\label{supp:eq:functional-delta-cluster}
\operatorname{Cl}_{C([\varepsilon,T];\mathbb R^s)}
(\mathcal D_n^F)
=
\mathcal G^F_{m,\alpha,\kappa;\varepsilon,T}
\qquad\text{a.s.},
\end{equation}
where
\begin{equation}
\label{supp:eq:functional-delta-set}
\mathcal G^F_{m,\alpha,\kappa;\varepsilon,T}
:=
\left\{
t\mapsto
-t^{\kappa-1}DF(\theta_0)q_0
\left\langle
Q_t^{(m,\alpha,\kappa)},\xi^{\otimes m}
\right\rangle
:
\|\xi\|_{\mathfrak H_{\mathbb R}}\le1
\right\}.
\end{equation}
In particular, the effective leading direction is
\[
DF(\theta_0)V^{-1}\mathcal J\in\mathbb R^s.
\]
\end{proposition}

\begin{proof}
Work on a probability-one event on which Proposition~4.1 of the main
paper gives consistency and Theorem~5.2 holds on both
\([\varepsilon,T]\) and \([\varepsilon/2,T+1]\).

Put \(I_n:=\{\lfloor n\varepsilon\rfloor,\ldots,\lceil nT\rceil\}\).
Define \(\delta_n:=\max_{k\in I_n}\|\widehat\theta_k-\theta_0\|\).
Since \(\varepsilon>0\), we have
\(\min I_n=\lfloor n\varepsilon\rfloor\to\infty\);
consistency therefore gives \(\delta_n:=\max_{k\in I_n}\|\widehat\theta_k-\theta_0\|\to0\).
Indeed, for every \(\eta>0\), the inequality
\(\|\widehat\theta_k-\theta_0\|<\eta\) holds for all sufficiently
large \(k\), and hence for every \(k\in I_n\) once \(n\) is
sufficiently large.

Choose \(\delta_0>0\) with
\(\overline B(\theta_0,\delta_0)\subset\Theta_0\).
For \(0\le\delta\le\delta_0\), define
\(\omega_F(\delta):=
\sup_{\theta\in\overline B(\theta_0,\delta)}
\|DF(\theta)-DF(\theta_0)\|_{\mathrm{op}}\).
Continuity of \(DF\) at \(\theta_0\) gives
\(\omega_F(\delta)\to0\) as \(\delta\downarrow0\).

Let \(\widetilde{\mathcal D}_n\) denote the normalized learning curve
on \([\varepsilon/2,T+1]\); its restriction to \([\varepsilon,T]\)
is \(\mathcal D_n\). For all sufficiently large \(n\),
\(I_n/n\subset[\varepsilon/2,T+1]\), because
\(\lfloor n\varepsilon\rfloor/n\ge\varepsilon-1/n>\varepsilon/2\)
and \(\lceil nT\rceil/n\le T+1/n<T+1\).
For \(k\in I_n\), the mesh values satisfy
\(\widetilde{\mathcal D}_n(k/n)
=(\widehat\theta_k-\theta_0)/r_n\).
Relative compactness of \(\{\widetilde{\mathcal D}_n\}\) implies
boundedness in the uniform norm. Hence there is a finite random
constant \(C=C(\omega)\) such that
\(\max_{k\in I_n}\|\widehat\theta_k-\theta_0\|/r_n\le C\)
for all sufficiently large \(n\).

For large \(n\), \(\delta_n<\delta_0\), so all segments joining
\(\theta_0\) to \(\widehat\theta_k\), \(k\in I_n\), lie in
\(\Theta_0\). Define
\(R_k:=F(\widehat\theta_k)-F(\theta_0)
-DF(\theta_0)(\widehat\theta_k-\theta_0)\).
The integral Taylor formula gives
\[
R_k
=
\int_0^1
\bigl[
DF\bigl(\theta_0+u(\widehat\theta_k-\theta_0)\bigr)
-DF(\theta_0)
\bigr]
(\widehat\theta_k-\theta_0)\,du.
\]
Since \(\|u(\widehat\theta_k-\theta_0)\|\le\delta_n\) for
\(0\le u\le1\), we obtain
\(\|R_k\|\le\omega_F(\delta_n)\|\widehat\theta_k-\theta_0\|\),
uniformly over \(k\in I_n\).
On the chosen event, \(C(\omega)<\infty\) and
\(\delta_n(\omega)\to0\); therefore
\begin{equation}
\label{supp:eq:mesh-delta-remainder}
\max_{k\in I_n}\frac{\|R_k\|}{r_n}
\le C(\omega)\omega_F(\delta_n)
\longrightarrow0
\qquad\text{a.s.}
\end{equation}

The mesh values defining \(\mathcal D_n^F\) and
\(DF(\theta_0)\mathcal D_n\) differ by \(R_k/r_n\) for \(k\in I_n\).
Linear interpolation commutes with \(DF(\theta_0)\) and, being a
convex combination of adjacent values, does not increase the
maximum norm of these differences. Thus
\eqref{supp:eq:mesh-delta-remainder} yields
\begin{equation}
\label{supp:eq:uniform-delta-remainder}
\sup_{\varepsilon\le t\le T}
\|\mathcal D_n^F(t)-DF(\theta_0)\mathcal D_n(t)\|
\longrightarrow0
\qquad\text{a.s.}
\end{equation}

Consider the continuous linear map
\(\mathcal A:C([\varepsilon,T];\mathbb R^p)
\to C([\varepsilon,T];\mathbb R^s)\), defined by
\((\mathcal A h)(t):=DF(\theta_0)h(t)\).
By Theorem~5.2 of the main paper, \(\{\mathcal D_n\}\) is relatively
compact with cluster set
\(\mathcal G_{m,\alpha,\kappa;\varepsilon,T}\).
Continuity of \(\mathcal A\) and
\eqref{supp:eq:uniform-delta-remainder} therefore imply relative
compactness of \(\{\mathcal D_n^F\}\).

To identify its cluster set, suppose that
\(\mathcal D_{n_j}^F\to g\) uniformly.
Relative compactness of \(\{\mathcal D_n\}\) allows passage to a
further subsequence along which \(\mathcal D_{n_j}\to h\) uniformly
for some \(h\in\mathcal G_{m,\alpha,\kappa;\varepsilon,T}\).
Equation~\eqref{supp:eq:uniform-delta-remainder} then gives
\(g=\mathcal A h\).
Conversely, every \(h\in\mathcal G_{m,\alpha,\kappa;\varepsilon,T}\)
is the uniform limit of a subsequence \(\mathcal D_{n_j}\), and
\eqref{supp:eq:uniform-delta-remainder} implies
\(\mathcal D_{n_j}^F\to\mathcal A h\) uniformly. Hence
\[
\operatorname{Cl}_{C([\varepsilon,T];\mathbb R^s)}
(\mathcal D_n^F)
=
\mathcal A\mathcal G_{m,\alpha,\kappa;\varepsilon,T}
=
\mathcal G^F_{m,\alpha,\kappa;\varepsilon,T}.
\]
The last equality follows by applying \(\mathcal A\) to the explicit
cluster set in Theorem~5.2 of the main paper, yielding
\eqref{supp:eq:functional-delta-set}. This proves
\eqref{supp:eq:functional-delta-cluster}.
\end{proof}

\begin{corollary}[Prediction functional]
\label{supp:cor:prediction-delta}
Under the assumptions of
Proposition~\ref{supp:prop-pathwise-delta}, fix a feature vector \(u\)
such that \(\theta\mapsto f_\theta(u)\) is continuously
differentiable on \(\Theta_0\), and put
\[
c_u
:=
\nabla_\theta f_{\theta_0}(u)^\top
V^{-1}\mathcal J.
\]
Let \(\mathcal P_n^{(u)}\) be the linear interpolation on
\([\varepsilon,T]\) of
\[
\mathcal P_n^{(u)}(k/n)
:=
\frac{
f_{\widehat\theta_k}(u)-f_{\theta_0}(u)
}{r_n}.
\]
Then \(\{\mathcal P_n^{(u)}\}\) is almost surely relatively compact
in \(C([\varepsilon,T];\mathbb R)\), and
\[
\operatorname{Cl}_{C([\varepsilon,T];\mathbb R)}
(\mathcal P_n^{(u)})
=
\left\{
t\mapsto
-t^{\kappa-1}c_u
\left\langle
Q_t^{(m,\alpha,\kappa)},\xi^{\otimes m}
\right\rangle
:
\|\xi\|_{\mathfrak H_{\mathbb R}}\le1
\right\}
\qquad\text{a.s.}
\]
\end{corollary}

\begin{proof}
Apply Proposition~\ref{supp:prop-pathwise-delta} with
\[
F(\theta)=f_\theta(u).
\]
\end{proof}

\begin{remark}[Parity of the prediction-error cluster set]
\label{supp:rem:prediction-parity}
Assume \(c_u\ne0\). By the scaling relation for
\(Q_t^{(m,\alpha,\kappa)}\) and Remark~4.4 of the main paper, for
every fixed \(t>0\) the pointwise cluster set of
\(\{\mathcal P_n^{(u)}(t)\}\) is
\[
-c_u t^{-\alpha m/2}
[0,\Lambda_{m,\alpha,\kappa}]
\]
when \(m\) is even, whereas for odd \(m\) it is the symmetric interval
\[
\left[
-|c_u|t^{-\alpha m/2}\Lambda_{m,\alpha,\kappa},
\,
|c_u|t^{-\alpha m/2}\Lambda_{m,\alpha,\kappa}
\right].
\]
Thus, for even \(m\), all nonzero pointwise cluster values of the
normalized prediction error have the same sign; excursions in the
opposite direction are negligible on the LIL scale. This is a
pathwise statement about the LIL-scale cluster geometry and does not
assert a nonzero mean or finite-sample bias.
\end{remark}

\begin{remark}[Degenerate delta-method direction]
\label{supp:rem:delta-degenerate}
If
\(
DF(\theta_0)V^{-1}\mathcal J=0,
\)
then
\(
\mathcal G^F_{m,\alpha,\kappa;\varepsilon,T}
=
\{0\}.
\)
Since \(\{\mathcal D_n^F\}\) is relatively compact and has the unique
cluster point \(0\), it follows that
\[
\sup_{\varepsilon\le t\le T}
\|\mathcal D_n^F(t)\|
\longrightarrow0
\qquad\text{almost surely}.
\]
The proposition does not identify the next nonzero scale in this
degenerate case.
\end{remark}

\section{Functional LIL for the excess population risk}
\label{supsec:functional-excess-risk}

We now record the functional excess-risk statement referred to after
Theorem~6.3 of the main paper. Throughout this section \(m\ge2\).
Fix \(0<\varepsilon<T\), and recall that
\(\mathcal D_n\) denotes the normalized learning curve of
Theorem~5.2 of the main paper,
\[
\mathcal D_n(k/n)
=
\frac{\widehat\theta_k-\theta_0}{r_n}.
\]
Also recall
\(
q_0=V^{-1}\mathcal J.
\)

For all sufficiently large \(n\), let
\(\mathcal R_n^{\mathrm{exc}}\) be the linear interpolation on
\([\varepsilon,T]\) of the mesh values
\[
\mathcal R_n^{\mathrm{exc}}(k/n)
:=
\frac{
R(\widehat\theta_k)-R(\theta_0)
}{r_n^2}.
\]

\begin{proposition}[Functional excess-risk cluster set]
\label{supp:prop-functional-excess-risk}
Assume the assumptions of Theorem~5.2 of the main paper. Then
\(\{\mathcal R_n^{\mathrm{exc}}\}\) is almost surely relatively
compact in \(C([\varepsilon,T];\mathbb R)\), and
\begin{equation}
\label{supp:eq:functional-excess-risk-cluster}
\operatorname{Cl}_{C([\varepsilon,T];\mathbb R)}
(\mathcal R_n^{\mathrm{exc}})
=
\mathcal G^{\mathrm{exc}}_{m,\alpha,\kappa;\varepsilon,T}
\qquad\text{a.s.},
\end{equation}
where
\begin{equation}
\label{supp:eq:functional-excess-risk-set}
\mathcal G^{\mathrm{exc}}_{m,\alpha,\kappa;\varepsilon,T}
:=
\left\{
t\mapsto
\frac12
\mathcal J^\top V^{-1}\mathcal J\,
t^{2\kappa-2}
\left\langle
Q_t^{(m,\alpha,\kappa)},\xi^{\otimes m}
\right\rangle^2
:
\|\xi\|_{\mathfrak H_{\mathbb R}}\le1
\right\}.
\end{equation}
\end{proposition}

\begin{proof}
Work on a probability-one event on which Proposition~4.1 of the main
paper gives consistency and Theorem~5.2 holds on both
\([\varepsilon,T]\) and \(J:=[\varepsilon/2,T+1]\).
Let \(\widetilde{\mathcal D}_n\) denote the normalized learning curve
on \(J\), whose restriction to \([\varepsilon,T]\) is
\(\mathcal D_n\).

Put \(I_n:=\{\lfloor n\varepsilon\rfloor,\ldots,\lceil nT\rceil\}\),
\(h_k:=\widehat\theta_k-\theta_0\), and
\(\delta_n:=\max_{k\in I_n}\|h_k\|\).
As in the proof of Proposition~\ref{supp:prop-pathwise-delta},
\(\min I_n\to\infty\) and \(I_n/n\subset J\) for large \(n\).
Consistency and relative compactness on \(J\) therefore give
\(\delta_n\to0\) and
\(\max_{k\in I_n}\|h_k\|/r_n\le C(\omega)<\infty\)
for all sufficiently large \(n\).

Let \(Q_V(x):=\tfrac12x^\top Vx\), and choose \(\delta_0>0\) with
\(\overline B(\theta_0,\delta_0)\subset\Theta_0\).
By \textup{(L2)}, \(\nabla^2R(\theta)=V(\theta)\), whose continuity
was established in Section~\ref{supsec:weighted-ERM}.
For \(0\le\delta\le\delta_0\), define
\(\omega_V(\delta):=\sup_{\|\theta-\theta_0\|\le\delta}
\|V(\theta)-V\|_{\mathrm{op}}\).
Then \(\omega_V(\delta)\to0\) as \(\delta\downarrow0\).

For large \(n\), \(\delta_n<\delta_0\). Since
\(\nabla R(\theta_0)=0\), the remainder
\(E_k:=R(\widehat\theta_k)-R(\theta_0)-Q_V(h_k)\) satisfies
\[
E_k
=
\int_0^1
(1-u)\,h_k^\top[V(\theta_0+uh_k)-V]h_k\,du,
\qquad k\in I_n.
\]
Thus \(|E_k|\le\tfrac12\omega_V(\delta_n)\|h_k\|^2\), and
\begin{equation}
\label{supp:eq:risk-quadratic-mesh}
\max_{k\in I_n}\frac{|E_k|}{r_n^2}
\le\tfrac12 C(\omega)^2\omega_V(\delta_n)
\longrightarrow0
\qquad\text{a.s.}
\end{equation}

Let \(\mathcal Q_n\) be the linear interpolation of the values
\(Q_V(h_k/r_n)\), restricted to \([\varepsilon,T]\).
The mesh differences between the excess-risk trajectory and
\(\mathcal Q_n\) equal \(E_k/r_n^2\).
As in the proof of Proposition~\ref{supp:prop-pathwise-delta},
linear interpolation does not increase their maximum absolute value.
Hence the bound in \eqref{supp:eq:risk-quadratic-mesh} also controls
\(\|\mathcal R_n^{\mathrm{exc}}-\mathcal Q_n\|_\infty\).
Unlike \(DF(\theta_0)\), however, \(Q_V\) does not in general commute
with interpolation, so an additional error must be controlled.

For a cell used by the interpolation, put
\(x=h_k/r_n\), \(y=h_{k+1}/r_n\), and \(\lambda=nt-k\).
At points of this cell in \([\varepsilon,T]\),
\(\mathcal D_n(t)=(1-\lambda)x+\lambda y\) and
\(\mathcal Q_n(t)=(1-\lambda)Q_V(x)+\lambda Q_V(y)\).
Consequently,
\[
\mathcal Q_n(t)-Q_V(\mathcal D_n(t))
=
\tfrac12\lambda(1-\lambda)(y-x)^\top V(y-x).
\]
Since \(0\le\lambda\le1\), the absolute value of this difference
is at most \(\tfrac18\|V\|_{\mathrm{op}}\|y-x\|^2\).
Relative compactness of \(\{\widetilde{\mathcal D}_n\}\) gives
a common pathwise modulus of continuity
\(\eta_\omega(\delta)\to0\) as \(\delta\downarrow0\).
All relevant cell endpoints lie in \(J\) for large \(n\), so
\(\|y-x\|\le\eta_\omega(1/n)\to0\), uniformly over these cells.
Together with \eqref{supp:eq:risk-quadratic-mesh}, this yields
\begin{equation}
\label{supp:eq:uniform-risk-quadratic}
\|\mathcal R_n^{\mathrm{exc}}-Q_V(\mathcal D_n(\cdot))\|_\infty
\le\tfrac12 C(\omega)^2\omega_V(\delta_n)
+\tfrac18\|V\|_{\mathrm{op}}\eta_\omega(1/n)^2
\longrightarrow0
\qquad\text{a.s.}
\end{equation}

The map \(\mathcal Q:h\mapsto Q_V(h(\cdot))\) is continuous from
\(C([\varepsilon,T];\mathbb R^p)\) to
\(C([\varepsilon,T];\mathbb R)\).
Theorem~5.2 of the main paper and
\eqref{supp:eq:uniform-risk-quadratic} therefore imply relative
compactness of \(\{\mathcal R_n^{\mathrm{exc}}\}\).
The same subsequence argument as in
Proposition~\ref{supp:prop-pathwise-delta} identifies its cluster set
as \(\mathcal Q\mathcal G_{m,\alpha,\kappa;\varepsilon,T}\).

For a member
\(h(t)=-t^{\kappa-1}q_0
\langle Q_t^{(m,\alpha,\kappa)},\xi^{\otimes m}\rangle\)
of \(\mathcal G_{m,\alpha,\kappa;\varepsilon,T}\), we have
\[
(\mathcal Qh)(t)
=
\tfrac12\mathcal J^\top V^{-1}\mathcal J\,
t^{2\kappa-2}
\langle Q_t^{(m,\alpha,\kappa)},\xi^{\otimes m}\rangle^2,
\]
because \(q_0^\top Vq_0=\mathcal J^\top V^{-1}\mathcal J\).
Thus \(\mathcal Q\mathcal G_{m,\alpha,\kappa;\varepsilon,T}
=\mathcal G^{\mathrm{exc}}_{m,\alpha,\kappa;\varepsilon,T}\),
which proves \eqref{supp:eq:functional-excess-risk-cluster}.
\end{proof}

\begin{remark}[Pointwise envelope and parity]
\label{supp:rem:excess-risk-parity}
Let
\[
H=1-\kappa-\frac{\alpha m}{2}.
\]
By the scaling relation
\[
\sup_{\|\xi\|_{\mathfrak H_{\mathbb R}}\le1}
\left|
\left\langle
Q_t^{(m,\alpha,\kappa)},\xi^{\otimes m}
\right\rangle
\right|
=
t^H\Lambda_{m,\alpha,\kappa},
\]
evaluation of
\eqref{supp:eq:functional-excess-risk-cluster} at any fixed
\(t>0\) gives the pointwise cluster interval
\[
\left[
0,\,
\frac12
\mathcal J^\top V^{-1}\mathcal J\,
t^{-\alpha m}
\Lambda_{m,\alpha,\kappa}^2
\right].
\]
In particular, its upper endpoint is exactly the limsup constant in
equation~(6.6) of the main paper. The parity distinction present in
the parameter endpoint cluster set disappears after the quadratic
map: both the one-sided even-\(m\) parameter cluster set and the
symmetric odd-\(m\) cluster set are mapped onto the same nonnegative
pointwise interval after squaring.
\end{remark}

\section{Continuity and bounds for the weighting objective}
\label{supsec:weighting-objective}

For \(m\ge2\) and \(\alpha m<1\), recall
\[
\mathcal D_{m,\alpha}
=
\left(-\infty,\frac{1-\alpha m}{2}\right),
\qquad
\mathcal W_{m,\alpha}(\kappa)
=
(1-\kappa)
\sqrt{m!I_{m,\alpha,\kappa}}\,
\Lambda_{m,\alpha,\kappa},
\]
where
\[
I_{m,\alpha,\kappa}
=
\frac{2\mathrm B(1-\kappa,1-\alpha m)}
     {2-2\kappa-\alpha m}.
\]

\begin{proposition}[Continuity and bounds]
\label{supp:prop-weighting-objective}
For fixed \(m\ge2\) and \(\alpha m<1\), the map
\(\kappa\mapsto\mathcal W_{m,\alpha}(\kappa)\) is continuous on
\(\mathcal D_{m,\alpha}\). Moreover,
\begin{equation}
\label{supp:eq:weighting-bounds}
\underline{\mathcal W}_{m,\alpha}(\kappa)
\le
\mathcal W_{m,\alpha}(\kappa)
\le
(1-\kappa)\sqrt{I_{m,\alpha,\kappa}},
\qquad
\kappa\in\mathcal D_{m,\alpha},
\end{equation}
where
\[
\underline{\mathcal W}_{m,\alpha}(\kappa)
:=
(1-\kappa)
\left(
\frac{2-\alpha}{2(1-\alpha)}
\right)^{m/2}
\times
\sum_{j=0}^{m}\binom mj
\mathrm B\!\left(
1-\kappa+(1-\alpha)j,\,
1+(1-\alpha)(m-j)
\right).
\]
\end{proposition}

\begin{proof}
The beta-function formula shows that
\(I_{m,\alpha,\kappa}\) is positive and continuous on
\(\mathcal D_{m,\alpha}\). Write
\[
Q^\kappa:=Q_1^{(m,\alpha,\kappa)},
\qquad
I_\kappa:=I_{m,\alpha,\kappa},
\]
and define, for \(\kappa,\lambda\in\mathcal D_{m,\alpha}\),
\[
J_{m,\alpha}(\kappa,\lambda)
:=
\int_0^1\!\!\int_0^1
y^{-\kappa}z^{-\lambda}
|y-z|^{-\alpha m}\,dy\,dz.
\]
Thus \(J_{m,\alpha}(\kappa,\kappa)=I_\kappa\).

Fix a compact \(K\subset\mathcal D_{m,\alpha}\) and let
\(\kappa_+:=\max K\). For \(\kappa,\lambda\in K\), the integrand is
bounded by
\[
y^{-\kappa_+}z^{-\kappa_+}|y-z|^{-\alpha m},
\]
which is integrable because
\(2\kappa_++\alpha m<1\). Dominated convergence therefore shows that
\(J_{m,\alpha}\) is continuous on \(K\times K\).

The kernel representation and the regularized Fourier calculation in\\
\citep[Lemma~10.14]{Moldavskaya2026WeightedLIL}, applied with weight
exponents \(\kappa\) and \(\lambda\), give
\[
\left\langle Q^\kappa,Q^\lambda\right\rangle
=
\frac{J_{m,\alpha}(\kappa,\lambda)}
     {m!\sqrt{I_\kappa I_\lambda}},
\qquad
\|Q^\kappa\|^2=\frac1{m!}.
\]
Consequently,
\[
\|Q^\kappa-Q^\lambda\|^2
=
\frac{2}{m!}
\left(
1-
\frac{J_{m,\alpha}(\kappa,\lambda)}
     {\sqrt{I_\kappa I_\lambda}}
\right)
\longrightarrow0
\qquad
(\lambda\to\kappa).
\]

Moreover,
$$
\left|
\Lambda_{m,\alpha,\kappa}
-
\Lambda_{m,\alpha,\lambda}
\right|
\le
\sup_{\|\xi\|_{\mathfrak H_{\mathbb R}}\le1}
\left|
\left\langle
Q^\kappa-Q^\lambda,\xi^{\otimes m}
\right\rangle
\right|
\le
\|Q^\kappa-Q^\lambda\|,
$$
because
\(\|\xi^{\otimes m}\|=\|\xi\|^m\le1\).
Hence \(\kappa\mapsto\Lambda_{m,\alpha,\kappa}\), and therefore
\(\kappa\mapsto\mathcal W_{m,\alpha}(\kappa)\), is continuous.

Finally, the two-sided estimate in
\cite[Lemma~7.2]{Moldavskaya2026WeightedLIL} can be written as
\[
\frac{
\underline{\mathcal W}_{m,\alpha}(\kappa)
}{
(1-\kappa)\sqrt{m!I_{m,\alpha,\kappa}}
}
\le
\Lambda_{m,\alpha,\kappa}
\le
\frac1{\sqrt{m!}}.
\]
Since \(1-\kappa>0\) on \(\mathcal D_{m,\alpha}\), multiplication by
\((1-\kappa)\sqrt{m!I_{m,\alpha,\kappa}}\) proves
\eqref{supp:eq:weighting-bounds}.
\end{proof}

The bounds in \eqref{supp:eq:weighting-bounds} show that
\(\mathcal W_{m,\alpha}(\kappa)\) is finite and strictly positive for
each fixed admissible \(\kappa\); they are not intended to localize its
minimizer, for which Section~8 of the main paper computes
\(\Lambda_{m,\alpha,\kappa}\) directly.

\section{Digamma estimates for rank-one weighting}
\label{supsec:rank-one-digamma}

Recall from equation~(6.12) of the main paper that
\[
D_\alpha(\kappa)
=
-\frac{2}{1-\kappa}
-\Psi(1-\kappa)
+\Psi(2-\kappa-\alpha)
+\frac{2}{2-2\kappa-\alpha}.
\]

\begin{lemma}[Sign estimates for \(D_\alpha\)]
\label{supp:lem-rank-one-signs}
For every \(\alpha\in(0,1)\),
\(
D_\alpha(\kappa)<0,
\quad \kappa\le0,
\)
and\\
\(
\lim_{\kappa\uparrow(1-\alpha)/2}
D_\alpha(\kappa)>0.
\)
\end{lemma}

\begin{proof}
Put
\(
\beta:=1-\alpha\in(0,1),
\quad
x:=1-\kappa.
\)
%%%%%%%%%%
For \(\kappa\le0\), \(x\ge1\). Since
\[
\Psi(x+\beta)-\Psi(x)
=
\int_0^1
t^{x-1}\frac{1-t^\beta}{1-t}\,dt,
\qquad
\frac{2}{x}
=
2\int_0^1 t^{x-1}\,dt,
\]
and
\[
\frac{2}{2x-\alpha}
=
\int_0^1 t^{x-1-\alpha/2}\,dt,
\]
where all three integrals converge for \(x\ge1\), we obtain
\[
D_\alpha(\kappa)
=
\int_0^1
t^{x-1}
\left[
\frac{1-t^\beta}{1-t}
-2+t^{-\alpha/2}
\right]dt.
\]%%%%%%%%%%
For \(0<t<1\),
\begin{equation}
\label{supp:eq:power-quotient-bound}
\frac{1-t^\beta}{1-t}
<
\beta t^{(\beta-1)/2}
=
\beta t^{-\alpha/2}.
\end{equation}
Indeed, writing \(t=e^{-2u}\), \(u>0\), reduces
\eqref{supp:eq:power-quotient-bound} to
\(
\sinh(\beta u)<\beta\sinh u,
\)
which follows from strict convexity of \(\sinh\) and
\(\sinh 0=0\). Hence
\[
D_\alpha(\kappa)
<
\frac{2-\alpha}{x-\alpha/2}-\frac{2}{x}
=
\frac{\alpha(1-x)}
{x(x-\alpha/2)}
\le0.
\]
The first inequality is strict also when \(x=1\), proving
\(D_\alpha(\kappa)<0\) for all \(\kappa\le0\).

For the right endpoint, put
\[
\delta:=\frac{1-\alpha}{2}\in\left(0,\frac12\right).
\]
The recurrence and reflection formulas for the digamma function give
\[
\Psi(1+\delta)-\Psi(1-\delta)
=
\frac1\delta-\pi\cot(\pi\delta).
\]
Therefore
$$
\lim_{\kappa\uparrow(1-\alpha)/2}D_\alpha(\kappa)
=
\frac1\delta-\pi\cot(\pi\delta)
-\frac{2\delta}{1-\delta}
=
2\delta
\left[
\sum_{n=1}^{\infty}\frac1{n^2-\delta^2}
-\frac1{1-\delta}
\right],
$$
where the second equality uses the partial-fraction expansion of
\(\cot\). Since \(0<\delta<1/2\),
\[
\frac1{n^2-\delta^2}
>
\frac1{(n-\delta)(n+1-\delta)}
\]
for every \(n\ge1\), and
\[
\sum_{n=1}^{\infty}
\frac1{(n-\delta)(n+1-\delta)}
=
\frac1{1-\delta}.
\]
The displayed limit is therefore strictly positive.
\end{proof}

Lemma~\ref{supp:lem-rank-one-signs} supplies the sign statements used
in the proof of Theorem~6.4 of the main paper.

\section{Verification of the learning examples}
\label{supsec:examples}

All three examples in Section~7 of the main paper are instances of
linear least squares. We first record the common verification. Let
\[
\ell(\theta;\zeta,Y)
=
\frac12(Y-\theta^\top\zeta)^2,
\qquad
A:=\mathbb E[\zeta\zeta^\top],
\qquad
b_\zeta:=\mathbb E[\zeta Y].
\]
If \(A\) is positive definite and
\(\theta_0=A^{-1}b_\zeta\), then
\[
R(\theta)-R(\theta_0)
=
\frac12(\theta-\theta_0)^\top
A(\theta-\theta_0),
\qquad
V=A.
\]
Moreover,
\[
\psi(\theta;\zeta,Y)
=
\zeta(\theta^\top\zeta-Y),
\qquad
\nabla_\theta\psi(\theta;\zeta,Y)
=
\zeta\zeta^\top.
\]
Let \(\Theta\) be compact with
\(\theta_0\in\operatorname{int}\Theta\), and put
\(C_\Theta:=\sup_{\theta\in\Theta}\|\theta\|\). The choices
\[
\Xi_0
=
\frac12(|Y|+C_\Theta\|\zeta\|)^2,
\qquad
\Xi_\ell=\Xi_1
=
\|\zeta\|(|Y|+C_\Theta\|\zeta\|),
\]
\[
\Xi_2=\|\zeta\|^2,
\qquad
\Xi_H=0,
\]
verify \textup{(L3)} whenever
\(\mathbb E|Y|^4+\mathbb E\|\zeta\|^4<\infty\).
The preceding identities also give \textup{(L2)} and
\textup{(L4)}, while positive definiteness of \(A\) gives the
well-separated minimizer in \textup{(L1)}. Since the weighted empirical
risk is measurable in the data and continuous in \(\theta\), a
measurable minimizer on \(\Theta\) exists by the measurable maximum
theorem, applied to \(-\widehat R_n\)
\cite[Theorem~18.19]{AliprantisBorder2006}.

Under \textup{(P1)}, \(|\rho_1|<1\). Indeed, if
\(|\rho_1|=1\), then
\(\mathbb E(X_t-\rho_1X_{t-1})^2=1-\rho_1^2=0\) for every \(t\).
Thus \(X_t=\rho_1X_{t-1}\) almost surely with the same deterministic
sign \(\rho_1\), so \(\rho(h)=\rho_1^h\) for every \(h\ge1\),
contradicting \(\rho(h)\to0\).

\subsection{One-step linear prediction}

\begin{proof}[Proof of Proposition~7.1 of the main paper]
Here
\(
\zeta=X_{t-1},
\quad
Y=X_t,
\quad
A=1,
\quad
b_\zeta=\rho_1.
\)
All Gaussian moments are finite, so the preceding calculation verifies
\textup{(L1)--(L4)}, with
\(\theta_0=\rho_1\) and \(V=1\).

At the target,
$$
g_t^{\mathrm{pred}}
=
X_{t-1}(\rho_1X_{t-1}-X_t)
=
\rho_1H_2(X_{t-1})
-
I_2(h_{t-1}\widetilde\otimes h_t),
$$
and hence the score belongs entirely to the second Wiener chaos.
Furthermore,
\[
\frac{
\Cov(g_0^{\mathrm{pred}},H_2(X_h))
}{
2\rho(h)^2
}
=
\rho_1
\left(\frac{\rho(h+1)}{\rho(h)}\right)^2
-
\frac{\rho(h+1)}{\rho(h)}.
\]
By \textup{(P1)}, the right-hand side converges to
\(\rho_1-1\). Thus
\(
m=2,
\quad
\mathcal J_{\mathrm{pred}}=\rho_1-1\ne0.
\)
\end{proof}

\subsection{Threshold-event classification}

\begin{proof}[Proof of Proposition~7.3 of the main paper]
Put
\[
\zeta_t=
\begin{pmatrix}
1\\ X_{t-1}
\end{pmatrix},
\qquad
Y_t=\mathbf 1_{\{X_t>b\}}.
\]
Then
\[
A=\mathbb E[\zeta_t\zeta_t^\top]=I_2
\]
and, using
\(\mathbb E[X_t\mathbf 1_{\{X_t>b\}}]=\varphi(b)\),
\[
b_\zeta
=
\mathbb E[\zeta_tY_t]
=
\begin{pmatrix}
p_b\\
\rho_1\varphi(b)
\end{pmatrix}.
\]
Since \(Y_t\) is bounded and \(X_{t-1}\) has moments of every order,
the common least-squares calculation verifies
\textup{(L1)--(L4)}, with
\[
\theta_0^{\mathrm{cls}}
=
\begin{pmatrix}
p_b\\
\rho_1\varphi(b)
\end{pmatrix},
\qquad
V=I_2.
\]

The score is
\[
g_t^{\mathrm{cls}}
=
\begin{pmatrix}
1\\ X_{t-1}
\end{pmatrix}
\left(
p_b+\rho_1\varphi(b)X_{t-1}
-\mathbf 1_{\{X_t>b\}}
\right).
\]
Since
\[
\Pi_1\mathbf 1_{\{X_t>b\}}
=
\varphi(b)X_t,
\]
its first coordinate satisfies
\[
\Pi_1g_{t,1}^{\mathrm{cls}}
=
\varphi(b)(\rho_1X_{t-1}-X_t)\ne0.
\]
Hence \(m=1\).

For the low-frequency coefficient,
\[
\Cov(g_{0,1}^{\mathrm{cls}},X_h)
=
\varphi(b)
\bigl(\rho_1\rho(h+1)-\rho(h)\bigr).
\]
Also,
\[
\mathbb E[X_{-1}X_h\mid X_0]
=
\rho(h+1)-\rho_1\rho(h)
+\rho_1\rho(h)X_0^2
\]
and
\[
\mathbb E[
X_0^2\mathbf 1_{\{X_0>b\}}
]
=
p_b+b\varphi(b).
\]
Consequently,
\[
\mathbb E[
X_{-1}X_h\mathbf 1_{\{X_0>b\}}
]
=
p_b\rho(h+1)
+\rho_1b\varphi(b)\rho(h),
\]
which gives
\[
\Cov(g_{0,2}^{\mathrm{cls}},X_h)
=
-\rho_1b\varphi(b)\rho(h).
\]
Dividing by \(\rho(h)\) and using
\(\rho(h+1)/\rho(h)\to1\), we obtain
\[
\mathcal J_{\mathrm{cls}}
=
-\varphi(b)
\begin{pmatrix}
1-\rho_1\\
\rho_1b
\end{pmatrix}.
\]
Its first component is nonzero because
\(\varphi(b)>0\) and \(\rho_1\ne1\).
\end{proof}

\subsection{Symmetry-induced rank transition}

\begin{proof}[Proof of Proposition~7.6 of the main paper]
Here
\[
\zeta=1,
\qquad
Y=(X_t-s)^2,
\qquad
A=1,
\qquad
b_\zeta=\mathbb E(X_t-s)^2=1+s^2.
\]
Since \(Y\) has moments of every order, the common least-squares
calculation verifies \textup{(L1)--(L4)}, with
\[
\theta_0^{(s)}=1+s^2,
\qquad
V=1.
\]
At the target,
\[
g_s(X_t)
=
1+s^2-(X_t-s)^2
=
2sH_1(X_t)-H_2(X_t).
\]
Thus, if \(s\ne0\),
\[
m=1,
\qquad
\mathcal J_s
=
\frac{\mathbb E[g_s(X_0)H_1(X_0)]}{1!}
=
2s.
\]
If \(s=0\), then \(g_0=-H_2\), and therefore
\[
m=2,
\qquad
\mathcal J_0
=
\frac{\mathbb E[g_0(X_0)H_2(X_0)]}{2!}
=
-\frac{\mathbb E[H_2(X_0)^2]}{2!}
=
-1.
\]
\end{proof}

\section{Numerical methods and additional results}
\label{supsec:numerics}

This section supplements Section~8 of the main paper with numerical
methods and additional checks. Reported values are rounded; agreement
under grid refinement is not a certified error bound.

\subsection{Deterministic computation of the weighting objective}
\label{supsubsec:numerical-method}

The implementation discretizes the reduced variational problem of
\cite{Moldavskaya2026WeightedLIL}. For \(J\) uniform cells of width
\(h=1/J\), put \(y_i=(i-\tfrac12)h\),
\(c_\alpha=2\Gamma(\alpha)\cos(\pi\alpha/2)\), and
\[
K_{ij}
=
c_\alpha\int_{(j-1)h}^{jh}|y_i-z|^{-\alpha}\,dz,
\qquad
w_i
=
\frac{(ih)^{1-\kappa}-((i-1)h)^{1-\kappa}}{1-\kappa}.
\]
The singular kernel is integrated exactly within each source cell;
this avoids evaluating its diagonal singularity. The normalization
uses midpoint quadrature:
\[
\|v\|_K^2=h\,v^\top Kv,
\qquad
\lambda_J(v)
=
\frac{\sum_{i=1}^J w_i(Kv)_i^m}
{\sqrt{m!c_\alpha^m I_{m,\alpha,\kappa}}},
\qquad
\|v\|_K=1.
\]
Starting from \(v_i^{(0)}\propto y_i^{-\kappa}\), normalized in
\(\|\cdot\|_K\), the shifted iteration is
\[
u_i^{(\ell)}
=
\frac{w_i}{h}(Kv^{(\ell)})_i^{m-1}
+\tfrac12v_i^{(\ell)},
\qquad
v^{(\ell+1)}
=
\frac{u^{(\ell)}}{\|u^{(\ell)}\|_K}.
\]
It stops when successive values of \(\lambda_J\) differ by less than
\(10^{-13}\), with a cap of \(30000\) iterations. The terminal value
\(\widehat\Lambda_J\) gives
\[
\widehat{\mathcal W}_J
=
(1-\kappa)\sqrt{m!I_{m,\alpha,\kappa}}\,
\widehat\Lambda_J.
\]
There is one deterministic starting vector per evaluation and no
random seed. The rank-one case provides the normalization check
\(\Lambda_{1,\alpha,\kappa}=1\).

For \(\bar\kappa=(1-\alpha m)/2\), SciPy's bounded scalar search uses
\([-0.5,\bar\kappa-10^{-6}]\), absolute tolerance \(10^{-7}\), and at
most \(120\) iterations; its result is denoted \(\widehat\kappa_J\).
The broad profiles at \(\alpha=0.2\) use \(J=800\) and \(81\)
equally spaced points on \([-0.8,\bar\kappa-0.003]\); the positive
profiles use \(J=1500\) and \(81\) points on
\([0,\bar\kappa-0.002]\). Thus the respective steps are
\((\bar\kappa+0.797)/80\) and \((\bar\kappa-0.002)/80\).
Convergence from one starting vector and a bounded scalar search do
not certify global optimality.

\subsection{Grid refinement and behavior near the boundary}
\label{supsubsec:grid-checks}

At \((\alpha,\kappa)=(0.2,0.1)\), grids \(J=800,1500,3000\) give
rank-two values \(0.680864,0.680863,0.680863\), respectively; all
three rank-three values round to \(0.354299\). The rank-one values
round to\\
 \(1.000001,1.000000,1.000000\).

At \(\alpha=0.2\), refining \(J=1500\) to \(3000\) changes the
rank-two and rank-three numerical minimizers from
\((0.131003,0.124985)\) to \((0.131032,0.125006)\).
The corresponding reductions relative to uniform weighting change
from \((0.205138\%,0.314533\%)\) to
\((0.205184\%,0.314589\%)\). At \(\kappa=-0.6\), the penalties change
from \((3.102524\%,5.038918\%)\) to
\((3.102523\%,5.038915\%)\).
Table~\ref{supp:tab:boundary-grids} records the checks near
\(\bar\kappa=0.05\); the rank-two and rank-three iterations required
\(13\) and \(15\) steps, respectively.

\begin{table}[t]
\centering
\small
\caption{Computed weighting objectives near the admissible boundary.
Columns specify the number \(J\) of cells; entries are rounded to six
decimal places. Dashes indicate cases not computed on that grid
in the notebook.}
\label{supp:tab:boundary-grids}
\begin{tabular}{@{}rrrrrr@{}}
\toprule
& \multicolumn{3}{c}{\((m,\alpha)=(2,0.45)\)}
& \multicolumn{2}{c}{\((m,\alpha)=(3,0.30)\)}\\
\cmidrule(lr){2-4}\cmidrule(l){5-6}
\(\kappa\)&1500&3000&6000&1500&3000\\
\midrule
0.0400&2.353599&2.353597&--&2.194260&2.194260\\
0.0450&2.353178&2.353176&--&2.193740&2.193739\\
0.0480&2.352946&2.352944&--&2.193445&2.193444\\
0.0490&2.352871&2.352870&2.352869&2.193349&2.193349\\
0.0495&2.352835&2.352833&2.352832&2.193302&2.193301\\
0.0498&2.352813&2.352811&2.352811&2.193274&2.193273\\
0.0499&2.352806&2.352804&2.352804&2.193265&2.193264\\
\bottomrule
\end{tabular}
\end{table}

The finest transition scans use \(J=3000\) and step
\(\Delta\alpha=0.00005\), on \([0.40450,0.40480]\) for \(m=2\) and
\([0.25770,0.25820]\) for \(m=3\).
Table~\ref{supp:tab:boundary-transition} shows the neighboring tested
values at which the optimizer changes from an interior point to
numerical contact with the imposed upper search limit. These are
algorithm- and grid-dependent observations, not certified intervals
for a critical parameter or a proof of a boundary infimum.

\begin{table}[t]
\centering
\small
\caption{Numerical transition at \(J=3000\).
The reported gaps are \(\bar\kappa-\widehat\kappa_J\), rounded to six
decimal places; a gap near \(10^{-6}\) reflects the search cutoff.}
\label{supp:tab:boundary-transition}
\begin{tabular}{@{}rrrrr@{}}
\toprule
\(m\)&\(\alpha_-\)&Gap at \(\alpha_-\)
&\(\alpha_+\)&Gap at \(\alpha_+\)\\
\midrule
2&0.40460&0.000004&0.40465&0.000001\\
3&0.25795&0.000051&0.25800&0.000001\\
\bottomrule
\end{tabular}
\end{table}

\subsection{Monte Carlo design and additional learning diagnostics}
\label{supsubsec:monte-carlo}

The Gaussian input is \(X_t=B_{H_f}(t+1)-B_{H_f}(t)\),
\(H_f=1-\alpha/2\), at unit time spacing, with covariance
\[
\rho(h)
=
\tfrac12\bigl(
|h+1|^{2H_f}-2|h|^{2H_f}+|h-1|^{2H_f}
\bigr).
\]
The Davies--Harte circulant embedding
\citep{DaviesHarte1987} generates paths with the prescribed covariance
using FFTs. The implementation rejects embedding eigenvalues below
\(-10^{-10}\) and clips those in \([-10^{-10},0)\) to zero.
Both experiments use \(M=120\) independent paths, \(N=2^{16}\),
\(n=2^j\), \(j=8,\ldots,16\), and weights \(a_t=t^{-\kappa}\).
Random numbers are generated by \texttt{numpy.random.default\_rng};
paths are shared across weights or target shifts within each
experiment.

For classification, \(\alpha=0.2\), \(b=0.5\),
\(\kappa\in\{-0.15,0,0.10,0.20\}\), and the seed is
\(\mathtt{20260830}\). With \(\zeta_t=(1,X_{t-1})^\top\) and
\(Y_t=\mathbf1_{\{X_t>b\}}\), cumulative sufficient statistics give
\[
\widehat\theta_n
=
\left(\sum_{t=1}^n a_t\zeta_t\zeta_t^\top\right)^{-1}
\sum_{t=1}^n a_t\zeta_tY_t.
\]
Here \(n\) counts training pairs, with an additional initial
observation \(X_0\). For the rank-transition experiment,
\(\alpha=0.4\), \(\kappa=0\),
\(s\in\{0,0.1,0.2,0.4,0.6\}\), and the seed is
\(\mathtt{20260831}\); the estimator is
\(\widehat\theta_n^{(s)}=n^{-1}\sum_{t=1}^n(X_t-s)^2\), with target
\(1+s^2\). 
Both computations use unconstrained least squares. Weighted ergodic
convergence of the sample moments and positive definiteness of \(A\)
give strong consistency; hence these estimators eventually coincide
almost surely with the constrained minimizers whenever
\(\theta_0\in\operatorname{int}\Theta\).

For replication \(q\), write
\(\Delta_n^{(q)}=\widehat\theta_n^{(q)}-\theta_0\). We report
\[
\operatorname{RMSE}_{\mathrm{tot}}(n)
=
\left(\frac1M\sum_{q=1}^M\|\Delta_n^{(q)}\|^2\right)^{1/2},
\qquad
\operatorname{RMSE}_{e}(n)
=
\left(\frac1M\sum_{q=1}^M(e^\top\Delta_n^{(q)})^2\right)^{1/2}.
\]
For classification, \(e_\parallel=q_0/\|q_0\|\) and
\(e_\perp=(-e_{\parallel,2},e_{\parallel,1})^\top\).
Slopes are least-squares fits of log RMSE, or log RMSE ratio, against
\(\log n\), with an intercept, using \(n=2^{12},\ldots,2^{16}\).
Total and directional diagnostics use the same simulated paths.

\begin{table}[t]
\centering
\small
\caption{Classification: log--log slopes over the final five dyadic
sizes and total RMSE at \(N=2^{16}\). The ratio is
\(\operatorname{RMSE}_\perp/\operatorname{RMSE}_\parallel\).
Slopes are rounded to four decimal places.}
\label{supp:tab:classification-slopes}
\begin{tabular}{@{}rrrrrr@{}}
\toprule
\(\kappa\)&Total&Parallel&Orthogonal&Ratio
&\(\operatorname{RMSE}_{\mathrm{tot}}(N)\)\\
\midrule
-0.15&-0.1826&-0.1752&-0.3383&-0.1631&0.066341\\
 0.00&-0.1800&-0.1723&-0.3397&-0.1674&0.066492\\
 0.10&-0.1778&-0.1698&-0.3406&-0.1708&0.066756\\
 0.20&-0.1751&-0.1667&-0.3417&-0.1750&0.067207\\
\bottomrule
\end{tabular}
\end{table}

Table~\ref{supp:tab:classification-slopes} and
Figure~\ref{supp:fig:classification-rmse} show similar total-error
slopes across weights, while the orthogonal component decays faster.
The total slopes remain more negative than the reference exponent
\(-\alpha/2=-0.1\), so this comparison is pre-asymptotic.
For the quadratic targets, terminal RMSEs in increasing order of
\(s\) are \(0.023400,0.034059,0.052500,0.094339,0.137762\);
the fitted slopes are those reported in Section~8.2 of the main paper.
These finite-sample second-moment diagnostics do not estimate an
almost-sure limsup constant.

\begin{figure}[t]
\centering
\includegraphics[width=0.70\textwidth]{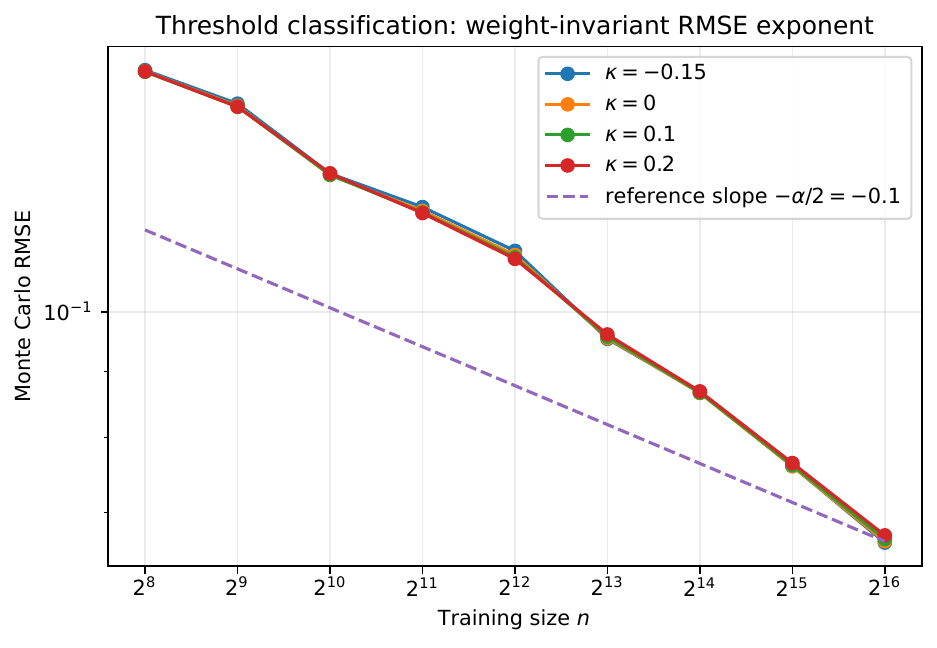}
\caption{Total parameter RMSE for threshold classification:
\(\alpha=0.2\), \(b=0.5\), \(M=120\), seed \(\mathtt{20260830}\),
and \(n=2^8,\ldots,2^{16}\). The dashed reference has slope \(-0.1\)
and is anchored at the unweighted RMSE at \(N=2^{16}\); it is not an
LIL boundary.}
\label{supp:fig:classification-rmse}
\end{figure}

\FloatBarrier